%% file: main.tex
\documentclass{article}

\usepackage{iclr2027/fancyhdr}
\usepackage{iclr2027/iclr2027_conference,times}
\usepackage[T1]{fontenc}
\usepackage{amsmath,amssymb,amsthm,mathtools}
\usepackage{array,booktabs}
\usepackage{microtype}
\usepackage{hyperref}
\usepackage{url}
\usepackage{placeins}
\hypersetup{colorlinks=true,citecolor=blue,linkcolor=blue,urlcolor=blue}

\newcommand{\E}{\mathbb E}
\newcommand{\R}{\mathbb R}
\newcommand{\ip}[2]{\left\langle #1,#2\right\rangle}
\newcommand{\norm}[1]{\left\lVert #1\right\rVert}
\DeclareMathOperator{\proj}{proj}
\DeclareMathOperator{\clip}{clip}
\DeclareMathOperator{\osc}{osc}

\newtheorem{theorem}{Theorem}[section]
\newtheorem{lemma}[theorem]{Lemma}
\newtheorem{proposition}[theorem]{Proposition}
\newtheorem{corollary}[theorem]{Corollary}
\theoremstyle{remark}
\newtheorem{remark}[theorem]{Remark}

\title{Scalable Cox Regression via Grouped Risk Sets and Sharper LogSumExp Rates}

\author{\makebox[\dimexpr\textwidth-2\tabcolsep\relax][c]{\normalfont\large
\begin{tabular}[t]{@{}c@{\hspace{3em}}c@{}}
Elizaveta Iashchinskaia & Egor Gladin\\
HSE University & HSE University
\end{tabular}}}

\begin{document}
\iclrfinaltrue
\maketitle
\pagestyle{plain}

\begin{abstract}
Motivated by the computational challenges of large-scale Cox regression, we study stochastic minimization of LogSumExp objectives over large sets. Mini-batch normalizer estimates generally yield biased gradients. We instead use a softplus surrogate that introduces one auxiliary scalar per normalizer and admits unbiased single-sample gradients. For smooth convex LogSumExp objectives, we prove an $O(T^{-1/2})$ averaged objective bound, improving the previous $T^{-1/4}$ analysis. With a strongly convex regularizer on the original variable, we also obtain a last-iterate squared-error rate of $\widetilde O(T^{-1})$ without strong convexity in the auxiliary variables. 
For Cox regression, the normalizers are defined over nested risk sets. We exploit this structure by grouping neighboring failures and sharing one auxiliary variable per group. The resulting compressed objective admits uniform score and curvature bounds that control the errors from grouping and softplus approximation. Together with the general optimization result, these bounds give a mean-square rate of $T^{-4/5}$, up to logarithmic factors, relative to the full Cox solution. The compressed estimator also matches the full estimator's asymptotic distribution. Experiments on synthetic and real survival datasets with slowly decreasing risk sets show a favorable performance relative to stochastic baselines.
\end{abstract}

\section{Introduction}
\label{sec:introduction}

The motivating problem for this work is large-scale Cox regression.  The Cox
proportional-hazards model relates covariates to a right-censored event time
without specifying the baseline hazard \citep{cox1972regression}.  Its partial
likelihood associates each failure with a LogSumExp normalizer over everyone
still at risk.  Nested risk sets admit efficient cumulative-sum and
coordinate-descent implementations for linear Cox models
\citep{simon2011regularization}, but these methods rely on full passes over the
data.  They are less natural when access is stochastic, data are distributed
or out of core, or a full pass is itself the resource to avoid.

The Cox bottleneck is an instance of a broader optimization problem:
large-set LogSumExp normalizers also arise in numerous applications, including large-class softmax models,
distributionally robust optimization, and contrastive learning
\citep{fagan2018unbiased,levy2020large,wei2026neuclip}.  Replacing a
normalizer by a mini-batch estimate generally biases the gradient of the
original objective.  Variational formulations instead introduce one auxiliary variable per
normalizer and admit unbiased single-sample gradients
\citep{bental1986expected,fagan2018unbiased}.  Motivated by the Cox problem,
we sharpen the analysis of the smooth, tunable softplus formulation of
\citet{gladin2025improved}.  For general smooth convex LogSumExp objectives we
obtain an $O(T^{-1/2})$ averaged rate.  With a strongly convex regularizer on
the original parameter, we also obtain an $O(\log T/T)$ last-iterate
squared-distance rate, even though the auxiliary variables are not strongly
convex.

The nested Cox structure supplies a second, equally important ingredient.
Nearby failure times have similar risk sets, so we group consecutive failures
and share one normalizer variable within each group. 

A single group--event--subject draw gives an
unbiased stochastic gradient of the resulting surrogate.
Importantly, the gradient
perturbation from grouping is second order in the relative risk-set change:
exponential reweighting within a group is multiplied by the variation of
neighboring risk-set gradients.  This sharper gradient control permits substantial
compression while controlling error relative to the full Cox optimizer and
preserving the estimator's first-order statistical behavior.

Our contributions are:
\begin{itemize}
 \item For smooth convex LogSumExp objectives, we establish an
 $O(T^{-1/2})$ averaged objective bound for the softplus method. 
 When a strongly convex regularizer on the original parameter is added, we further
 obtain an $O(\log T/T)$ last-iterate squared-distance bound without assuming
 strong convexity in the auxiliary variables. 
 \item For Cox regression, we give a deterministic compression theorem that
 controls the gradient and Hessian uniformly through grouping-independent
 measures of risk-weight dispersion.  The gradient error caused by grouping is
 second order in the within-group risk-set change, while the softplus error is
 first order in its approximation level.  These bounds transfer curvature
 from the full Cox objective and, together with the general optimization
 result, yield mean squared parameter error
 $O((\log T/T)^{4/5})$ relative to the full Cox solution.
 \item We prove that the compressed and full Cox estimators have the same
 first-order limit under vanishing approximation schedules, while using a
 sublinear number of auxiliary shifts.  Experiments on synthetic and real
 survival data evaluate the resulting accuracy--computation tradeoff against
 stochastic Cox baselines.
\end{itemize}

The general optimization results apply to smooth convex LogSumExp objectives,
our survival-analysis claims concern the convex linear Cox model.  This scope
lets us separate computational approximation from statistical model error and
gives proofs that can be checked directly.
Section~\ref{sec:optimization-main} develops the general optimization
guarantees.  Section~\ref{sec:cox-main} specializes the construction to Cox
regression, establishes the deterministic approximation properties, derives
the end-to-end computational rate, and gives the statistical transfer result.
The appendix contains all auxiliary lemmas and complete proofs.

\section{Related work}
\label{sec:related}

\paragraph{Scalable Cox optimization.}
Exact Cox solvers exploit nested risk sets and sparsity but require full-data
passes \citep{simon2011regularization}.  \citet{achab2015sgd} approximate each
risk-set expectation by MCMC inside a variance-reduced method. Their strongly
convex guarantee uses increasing MCMC accuracy and periodic full gradients.
Unbiased multilevel gradients provide an exact stochastic-composition route
\citep{blanchet2017unbiased}.  Optimal subsampling and one-step correction
instead approximate or recover the full-data estimator through a reduced
subject sample \citep{zhang2024subsampling,wang2024bigcox}.

Classical nested case-control estimators have large-sample guarantees under
correct Cox specification \citep{goldstein1992asymptotic}.  Sampled-risk-set
neural losses adapt this construction \citep{kvamme2019time}, while online
small-stratum objectives \citep{tarkhan2024online} and ordinary mini-batch
partial likelihood \citep{zeng2026minibatch} target related stochastic
criteria.  In particular, the mini-batch estimator generally optimizes a
batch-size-dependent population objective rather than the realized full-data
partial likelihood.  We instead retain a deterministic, uniformly controlled
approximation to the fixed full-data Cox gradient while compressing its outer
family of normalizers.  Learned normalizer predictors can also compress
per-example state \citep{wei2026neuclip}; our construction uses nested risk
sets and provides uniform gradient and Hessian guarantees.

\paragraph{LogSumExp optimization.}
The Cox loss is a finite sum of coupled compositional objectives of the type
studied by \citet{wang2022finite}.  Single-loop methods obtain near-optimal
rates when the outer functions are convex and nondecreasing
\citep{wang2025alexr}, whereas the direct logarithmic decomposition falls
outside that setting.  For the softplus formulation, where $\rho>0$ is the
approximation parameter and smaller values give a sharper approximation,
\citet{gladin2025improved} obtain
$O((\rho\sqrt T)^{-1}+\rho)$, which optimizes to $T^{-1/4}$; our
expected-smoothness analysis gives $O((\rho T)^{-1}+\rho)$ and hence
$O(T^{-1/2})$.

SCENT optimizes the exact exponential variational form and also obtains an
$O(T^{-1/2})$ averaged convex-objective rate \citep{wei2026geometry}.  Its
general guarantee depends on variance quantities accumulated along the
optimization trajectory, and
therefore does not directly provide an a priori bound with explicit constants.
In contrast, our bounds are expressed in terms of fixed problem parameters.  Our strongly convex result additionally gives an
$O(\log T/T)$ last-iterate squared-distance rate without uniform strong
convexity in the auxiliary directions.  Applied directly to Cox, variational
methods retain one shift per event; risk-set grouping reduces this state from
$m$ to $K$ and quantifies the resulting error relative to the full Cox target.

\section{Stochastic optimization of LogSumExp objectives}
\label{sec:optimization-main}

We first state the optimization results independently of Cox regression.
Consider $K$ distributions $\mu_k$ with weights $p_k>0$,
$\sum_{k=1}^Kp_k=1$, and write $\E_k$ for expectation with respect to
$X\sim\mu_k$.  The convex objective under consideration is
\begin{equation}
 J(\theta)=R(\theta)+\sum_{k=1}^Kp_k
 \log\E_k e^{L_k(X,\theta)}
 \label{eq:main-generic-objective}
\end{equation}
on a compact convex set $\Theta$.  Uniformly over the domain, assume that
$R$ and every $L_k$ are twice continuously differentiable,
$|L_k|\le B$, $\norm{\nabla L_k}\le G$,
$0\preceq\nabla^2L_k\preceq H I$,
$\norm{\nabla R}\le G_R$, and
$0\preceq\nabla^2R\preceq L_RI$. 
Define the normalized exponential moment
\[
 \bar\kappa=\max_k\sup_{\theta\in\Theta}
 \frac{\E_k e^{2L_k(X,\theta)}}
 {(\E_k e^{L_k(X,\theta)})^2}.
\]

For $0<\rho<1$, let
$h_\rho(u)=\rho^{-1}\log(1+\rho e^u)$. Building on the softplus construction of \citet{gladin2025improved}, we introduce the joint surrogate and its profiled counterpart:
\begin{equation}
 G_\rho(\theta,s)=R(\theta)+\sum_{k=1}^Kp_k
 \{s_k-1+\E_kh_\rho(L_k(X,\theta)-s_k)\},
 \qquad J_\rho(\theta)=\min_{s \in \mathbb R^K}G_\rho(\theta,s).
 \label{eq:main-generic-surrogate}
\end{equation}
Proposition~\ref{prop:generic-approximation}, based on the approximation
bound of \citet{gladin2025improved}, gives
$J+\rho/2+\log(1-\rho\bar\kappa)\le J_\rho\le J$ whenever
$\rho\bar\kappa<1$.  If $\rho\bar\kappa\le1/8$, all optimal
shifts lie in $[-B-1,B]^K$.

For $z=(\theta,s)$ and a gradient $g=(g_\theta,g_s)$, let
$P=\operatorname{diag}(I_d,p_1,\ldots,p_K)$ and use the matrix-induced norm
and its dual:
\[
 \norm z_P^2=z^\top Pz=\norm\theta^2+\sum_kp_ks_k^2,
 \qquad
 \norm g_{P^{-1}}^2=g^\top P^{-1}g=\norm{g_\theta}^2+
 \sum_k\frac{g_{s,k}^2}{p_k}.
\]
At each iteration, sample $k$ uniformly from $\{1,\ldots,K\}$ and
$X$ from its group distribution.  With
\[
 w=\frac1{\rho+\exp\{s_k-L_k(X,\theta)\}},
\]
an unbiased stochastic gradient of $G_\rho$ is
\[
 g_\theta=\nabla R(\theta)+Kp_kw\nabla L_k(X,\theta),
 \qquad g_{s,k}=Kp_k(1-w),
\]
with all other shift coordinates zero.  We apply the corresponding weighted
projected step on $\mathcal Z:=\Theta\times[-B-1,B]^K$:
\begin{equation}
 z_{t+1}=\arg\min_{z\in\mathcal Z}
 \left\{\eta\ip{g_t}{z}+\frac12\norm{z-z_t}_P^2\right\}.
 \label{eq:main-generic-update}
\end{equation}
Equivalently, the parameter uses an ordinary projected step, while the sampled
shift is clipped after $s_k\leftarrow s_k-\eta K(1-w)$; all other shifts stay
fixed.  Initialize $z_1$ anywhere in this domain.  Let
$z_\rho=(\theta_\rho,s_\rho)$ minimize $G_\rho$ there and define
\begin{align}
 C&=L_R+KH+\frac K4(G^2+1),
 &V&=2G_R^2+2K+8K\bar\kappa(G^2+1),
 \label{eq:main-generic-constants}\\
 D^2&=\operatorname{diam}(\Theta)^2+(2B+1)^2.
 \label{eq:main-generic-diameter}
\end{align}

The next two results cover the general convex regime with iterate averaging
and the strongly convex regime with last-iterate convergence, respectively.

\begin{theorem}
\label{thm:main-convex}
Set $\rho_T=T^{-1/2}$ and $\eta_T=(4C\sqrt T)^{-1}$, and suppose
$T\ge64\bar\kappa^2$.  For the averaged parameter
$\bar\theta_T=T^{-1}\sum_{t=1}^T\theta_t$,
\begin{equation}
 \E[J(\bar\theta_T)-\min_\Theta J]
 \le
 \frac{4CD^2}{\sqrt T}
 +\frac{V}{2C\sqrt T}
 +\frac{8\bar\kappa}{7\sqrt T}.
 \label{eq:main-convex-rate}
\end{equation}
\end{theorem}
Thus the softplus method attains an $O(T^{-1/2})$ objective rate for the
original LogSumExp problem.

\begin{theorem}
\label{thm:main-optimization}
Assume $0<\rho<1$, $\rho\bar\kappa\le1/8$, and that $J_\rho$ is
$\gamma$-strongly convex.  Define
\begin{equation}
 \mu=\left\{\frac{2(1+2G^2)}\gamma+16B+8\right\}^{-1}.
 \label{eq:main-generic-mu}
\end{equation}
For a constant step size
$0<\eta\le\min\{\rho/(4C),1/\mu\}$, the last iterate satisfies
\begin{equation}
 \E\norm{z_{T+1}-z_\rho}_P^2
 \le e^{-\mu\eta T}D^2+\frac{2\eta V}{\mu}.
 \label{eq:main-parameterized-rate}
\end{equation}
In particular, for $T\ge2$, set $\rho_T=8C\log T/(\mu T)$ and
$\eta_T=\rho_T/(4C)$.  If these choices satisfy the preceding conditions,
then
\begin{equation}
 \E\norm{\theta_{T+1}-\theta_{\rho_T}}^2
 \le\frac{D^2}{T^2}+\frac{4V\log T}{\mu^2T}.
 \label{eq:main-optimized-generic-rate}
\end{equation}
\end{theorem}

\begin{corollary}
\label{cor:main-generic-objective}
Let $\theta^\star$ minimize $J$ and assume that it is stationary, and define
$L_J=L_R+H+G^2$.  Under the scheduled choices of
Theorem~\ref{thm:main-optimization}, the last iterate satisfies
\begin{equation}
 \E[J(\theta_{T+1})-J(\theta^\star)]
 \le
 \frac{L_JD^2}{T^2}
 +\frac{4L_JV\log T}{\mu^2T}
 +\frac{16L_J\rho_T\bar\kappa}{7\gamma}
 =O\left(\frac{\log T}{T}\right).
 \label{eq:main-generic-objective-rate}
\end{equation}
\end{corollary}

If $R$ is $\lambda$-strongly convex, then so is $J_\rho$, and one may take
$\gamma=\lambda$.
Appendix~\ref{sec:generic} proves the two theorems and the corollary using the
expected-smoothness estimate of Lemma~\ref{lem:smoothness-noise}
and, under strong convexity, the restricted-secant inequality
of Lemma~\ref{lem:generic-rsi}.

\section{Cox regression with grouped risk sets}
\label{sec:cox-main}

We now introduce grouped Cox risk sets and 
specialize the construction of
Section~\ref{sec:optimization-main} to the resulting objective.

\subsection{Cox objective and risk-set grouping}

\paragraph{Full Cox objective.}

Consider $N$ subjects.  Subject $j$ has observed time $y_j$, event indicator
$\Delta_j\in\{0,1\}$, and covariate vector $x_j\in\R^d$.  For simplicity, we assume distinct
observed failure times.\footnote{The finite-data compression results accommodate
tied failures under the Breslow convention~\citep{breslow1974covariance}, as
does the experimental implementation (Appendix~\ref{sec:study27-protocol}).
The statistical transfer result assumes distinct failures.}
Let
$\mathcal E=\{j:\Delta_j=1\}$ and assume
$m=|\mathcal E|\ge1$. For $i\in\mathcal E$, define the risk set and its
size by
\[
 \mathcal R_i=\{j:y_j\ge y_i\},\qquad n_i=|\mathcal R_i|.
\]
For a coefficient vector $\beta$ in a compact convex set
$\mathcal B\subset\R^d$, write
\begin{equation}
 a_i(\beta)=\log\left(\frac1{n_i}\sum_{j\in\mathcal R_i}
 e^{\beta^\top x_j}\right),\qquad
 \mathcal L(\beta)=\frac\lambda2\norm\beta^2-\beta^\top\bar x_{\mathcal E}
 +\frac1m\sum_{i\in\mathcal E}a_i(\beta),
 \label{eq:main-exact-cox}
\end{equation}
where $\lambda\ge0$ is an optional ridge coefficient and
$\bar x_{\mathcal E}=m^{-1}\sum_{i\in\mathcal E}x_i$.  The factors
$1/n_i$ add only a constant to the usual negative log partial likelihood, so
they do not change its minimizer for a fixed $\lambda$ \citep{cox1972regression}.
Setting $\lambda=0$ recovers the unpenalized objective.
We call $\beta^\top x_j$ the linear predictor and $e^{\beta^\top x_j}$
the risk weight of subject $j$.

\paragraph{Compressing neighboring risk sets.}

Order failures by time and greedily partition them into consecutive groups
$\mathcal I_1,\ldots,\mathcal I_K$.  A group is extended for as long as
\begin{equation}
 \frac{\max_{i\in\mathcal I_k}n_i}{\min_{i\in\mathcal I_k}n_i}
 \le 1+\delta,
 \label{eq:main-grouping-rule}
\end{equation}
where $\delta>0$ is chosen by the user.  Equivalently, setting
$r_\delta:=\delta/(1+\delta)$ bounds the fraction of the largest risk set
that disappears within a group by $r_\delta$.  Let $m_k=|\mathcal I_k|$ and
$p_k=m_k/m$.  Conditional on group $k$, draw an event $I$ uniformly from
$\mathcal I_k$ and then a subject $J$ uniformly from $\mathcal R_I$.  We use
$\E_k$ for this two-stage expectation and set
\begin{equation}
 \Phi_k(\beta)=\log\E_k e^{\beta^\top x_J}
 =\log\left(\frac1{m_k}\sum_{i\in\mathcal I_k}e^{a_i(\beta)}\right).
 \label{eq:main-group-normalizer}
\end{equation}
Replacing the average of the $a_i$ in each group by $\Phi_k$ gives
the grouped objective
$\widetilde{\mathcal L}_\delta(\beta)
=\lambda\norm\beta^2/2-\beta^\top\bar x_{\mathcal E}
+\sum_kp_k\Phi_k(\beta)$.  The replacement is
exact for singleton groups.  More importantly, its gradient error is quadratic in the within-group risk-set change; this is formalized in
Theorem~\ref{thm:main-compression}.

For $0<\rho<1$, use the softplus function $h_\rho$ from
Section~\ref{sec:optimization-main}.  Introduce
one shift $s_k\in\R$ per group and the joint objective
\begin{equation}
 G_{\rho,\delta}(\beta,s)
 =\frac\lambda2\norm\beta^2-\beta^\top\bar x_{\mathcal E}
 +\sum_{k=1}^Kp_k\left[s_k-1+
 \E_k h_\rho(\beta^\top x_J-s_k)\right].
 \label{eq:main-joint-objective}
\end{equation}
Its profiled version is
\begin{equation}
 \mathcal L_{\rho,\delta}(\beta)
 =\min_{s\in\R^K}G_{\rho,\delta}(\beta,s).
 \label{eq:main-profile-objective}
\end{equation}
Thus $\mathcal L$, $\widetilde{\mathcal L}_\delta$, and
$\mathcal L_{\rho,\delta}$ denote the exact, grouped, and
grouped-softplus objectives, all with the same $\lambda$.
With $R(\beta)=\lambda\norm\beta^2/2-\beta^\top\bar x_{\mathcal E}$,
$L_k((I,J),\beta)=\beta^\top x_J$, and $\Theta=\mathcal B$, this is exactly
the generic construction in \eqref{eq:main-generic-surrogate} for the
two-stage group distributions defined above.
The scalar minimizers are unique, finite, and have strictly positive second derivatives.  Consequently the envelope and
implicit-function theorems justify differentiating the profile and taking the
Schur complement of the joint Hessian; details appear in
Appendix~\ref{sec:generic}.

\paragraph{Stochastic updates and cost model.}
\label{par:stochastic-updates-cost}
Assume $\norm{x_j}\le X$ and $|\beta^\top x_j|\le M$ for every subject and
$\beta\in\mathcal B$.  At iteration $t$, draw
$k_t\sim\operatorname{Unif}\{1,\ldots,K\}$, then
$i_t\sim\operatorname{Unif}(\mathcal I_{k_t})$, and finally
$j_t\sim\operatorname{Unif}(\mathcal R_{i_t})$.  With
\[
 \widehat x_{\mathcal E,t}=Kp_{k_t}x_{i_t},\qquad
 w_t=\frac1{\rho+\exp(s_{k_t,t}-\beta_t^\top x_{j_t})},
\]
we have $\E_t\widehat x_{\mathcal E,t}=\bar x_{\mathcal E}$.
One projected stochastic-gradient step is
\begin{align}
 \beta_{t+1}
 &=\proj_{\mathcal B}\!\left[
  \beta_t-\eta\{\lambda\beta_t
  +Kp_{k_t}(w_tx_{j_t}-x_{i_t})\}\right],
 \label{eq:main-beta-update}\\
 s_{k_t,t+1}
 &=\clip_{[-M-1,M]}\!\left[s_{k_t,t}-\eta K(1-w_t)\right],
 \label{eq:main-shift-update}
\end{align}
and all unsampled shifts stay fixed. The sampled Euclidean gradient of
\eqref{eq:main-joint-objective} has coefficient component
$\lambda\beta_t+Kp_{k_t}(w_tx_{j_t}-x_{i_t})$ and shift component
$g_{s,k_t}=Kp_{k_t}(1-w_t)$, with all other shift components zero.
This gradient estimator is unbiased. The weighted update divides the
sampled shift component by $p_{k_t}$, giving
\eqref{eq:main-shift-update}; see Appendix~\ref{sec:generic}.

After event times have been sorted and group boundaries stored, an iteration
uses two sampled covariate vectors and $O(d)$ arithmetic.  The optimizer state
is $O(d+K)$ beyond storage or streaming access to the data.  We assume constant
time access to a uniformly sampled member of an indexed risk set and an
inexpensive projection onto $\mathcal B$. 

\subsection{Gradient and curvature preservation}
\label{sec:compression}

We first introduce the two data quantities that govern the result.  The
normalized risk-weight second moment is
\begin{equation}
 \kappa=\sup_{\beta\in\mathcal B}\max_{i\in\mathcal E}
 \frac{n_i^{-1}\sum_{j\in\mathcal R_i}e^{2\beta^\top x_j}}
 {(n_i^{-1}\sum_{j\in\mathcal R_i}e^{\beta^\top x_j})^2}.
 \label{eq:main-kappa}
\end{equation}
Fix $q\in(0,1)$.  To control the risk weights of subjects leaving over a short
time interval, define
\begin{equation}
 L_q=1\vee\sup_{\beta\in\mathcal B}
 \max_{\substack{u,v\in\mathcal E,\ \mathcal R_v\subsetneq\mathcal R_u\\
                  n_v\ge qn_u}}
 \frac{|\mathcal R_u\setminus\mathcal R_v|^{-1}
 \sum_{j\in\mathcal R_u\setminus\mathcal R_v}e^{\beta^\top x_j}}
 {n_u^{-1}\sum_{j\in\mathcal R_u}e^{\beta^\top x_j}}.
 \label{eq:main-Lq}
\end{equation}
For fixed $\beta$ and $\mathcal R_u$, the maximum ranges over later nested
risk sets $\mathcal R_v$ that retain at least a $q$ fraction of
$\mathcal R_u$; the ratio compares the mean risk weight of the subjects who
leave between the two event times with the mean risk weight in
$\mathcal R_u$.
If the maximum has no admissible pair, it is omitted.
Both quantities are
defined independently of the grouping.  They are finite under bounded linear
predictors, but keeping them explicit replaces a worst-case exponential range
factor by the observed normalized dispersion and local risk-weight ratio.

\begin{theorem}
\label{thm:main-compression}
Suppose $r_\delta\le1-q$, $L_qr_\delta\le1/2$, and
$\rho\kappa\le1/9$.
Then, uniformly for $\beta\in\mathcal B$,
\begin{align}
 \norm{\nabla\mathcal L_{\rho,\delta}(\beta)-\nabla\mathcal L(\beta)}
 &\le X\{L_q^2r_\delta^2+4\rho\kappa\},
 \label{eq:main-gradient-bound}\\
 \nabla^2\mathcal L_{\rho,\delta}(\beta)
 &\succeq\nabla^2\mathcal L(\beta)
 -X^2\left\{\frac52L_q^2r_\delta^2+20\rho\kappa\right\}I.
 \label{eq:main-hessian-bound}
\end{align}
If $\nabla^2\mathcal L(\beta)\succeq\nu I$ on $\mathcal B$ for some
$\nu>0$, and if $\beta^\star$ and $\beta_{\rho,\delta}$ minimize
$\mathcal L$ and $\mathcal L_{\rho,\delta}$ on $\mathcal B$, respectively,
then
\begin{equation}
 \norm{\beta_{\rho,\delta}-\beta^\star}
 \le\frac X\nu\{L_q^2r_\delta^2+4\rho\kappa\}.
 \label{eq:main-minimizer-bound}
\end{equation}
\end{theorem}

For $\lambda>0$, the ridge term supplies the full-objective curvature
assumption with $\nu=\lambda$; when $\lambda=0$, this assumption depends on
the data and design.
The Hessian bound also shows that if the matrix error subtracted in
\eqref{eq:main-hessian-bound} is at most $\nu/2$, then
$\mathcal L_{\rho,\delta}$ is $\nu/2$-strongly convex.
The quadratic grouping error comes from multiplying the
$O(L_qr_\delta)$ change in event weights by the
$O(XL_qr_\delta)$ variation of event gradients within a group.
Softplus profiling contributes $O(\rho\kappa)$, and the Hessian
comparison retains the nonnegative covariance term.
Appendix~\ref{sec:cox} proves the gradient, curvature, and minimizer bounds.

\subsection{End-to-end computational rate}
\label{sec:cox-optimization-main}

We now combine the generic last-iterate guarantee of
Theorem~\ref{thm:main-optimization} with the Cox gradient and curvature bounds in
Theorem~\ref{thm:main-compression}.

\begin{corollary}
\label{cor:main-cox-rate}
Assume $\nabla^2\mathcal L(\beta)\succeq\nu I$ on $\mathcal B$ for some
$\nu>0$, fix $q\in(0,1)$, and suppose every event risk set contains at least
$cN$ subjects for a fixed $c\in(0,1)$.  Let the problem constants be fixed as
$T$ varies.  For all sufficiently large declared horizons $T$, use maximal
grouping and the schedule
\begin{equation}
 \delta_T=\left(\frac{\log T}{T}\right)^{1/5},\qquad
 \rho_T=\frac{8A_{K_T}\log T}{\mu T},\qquad
 \eta_T=\frac{2\log T}{\mu T},
 \label{eq:main-cox-schedule}
\end{equation}
where
$A_K=\lambda+K(X^2+1)/4$ and
$\mu=\{4(1+2X^2)/\nu+16M+8\}^{-1}$.
Initialize $(\beta_1,s_1)$ in
$\mathcal B\times[-M-1,M]^{K_T}$.  Then, conditional on the data, the last
iterate of
\eqref{eq:main-beta-update}--\eqref{eq:main-shift-update} satisfies
\begin{equation}
 \E\norm{\beta_{T+1}-\beta^\star}^2
 =O\left((\log T/T)^{4/5}\right).
 \label{eq:main-four-fifths}
\end{equation}
When $\beta^\star$ is interior, the exact-objective error has the same order,
and the auxiliary state is
$K_T=O((T/\log T)^{1/5})$ until singleton groups cause finite-data saturation.
\end{corollary}

Indeed, maximal grouping gives $K_T=O(\delta_T^{-1})$, so the stochastic
error is $O(\log T/(\delta_TT))$, whereas the squared grouping bias is
$O(\delta_T^4)$.  Balancing these terms gives \eqref{eq:main-four-fifths};
Appendix~\ref{sec:cox} gives the exact finite-time bound and all smallness
conditions.

\subsection{First-order statistical efficiency}
\label{sec:statistics-main}

We now let the data set vary with sample size $N$ and set $\lambda=0$ in
all objectives on a fixed compact convex set $\mathcal B$.  Fix
$q\in(0,1)$ and assume $\max_{1\le j\le N}\norm{x_j}\le X$ uniformly in
$N$.  Let $\widehat\beta_N$ minimize the full empirical objective
$\mathcal L_N$, and let $\widehat\beta_N^{\rho,\delta}$ minimize its
grouped-softplus counterpart.  Classical conditions for the Cox model yield
asymptotic normality of the maximum partial-likelihood estimator
\citep{tsiatis1981large,andersen1982cox}.  We take the asymptotic normality of
the full Cox estimator as given and isolate the additional conditions under
which it transfers to the grouped-softplus estimator.

\begin{theorem}
\label{thm:main-statistics}
Suppose $\sqrt N(\widehat\beta_N-\beta_0)
\Longrightarrow\mathcal N(0,\mathcal I(\beta_0)^{-1})$ for a
positive-definite information matrix $\mathcal I(\beta_0)$.
Assume that, with probability tending to one, $\widehat\beta_N$ and
$\widehat\beta_N^{\rho,\delta}$ lie in a fixed convex neighborhood
$U\subset\mathcal B$ of $\beta_0$ and
$\nabla^2\mathcal L_N(\beta)\succeq\nu_0I$ on $U$ for a fixed $\nu_0>0$.
Let $L_{q,N}$ and $\kappa_N$ denote \eqref{eq:main-Lq} and
\eqref{eq:main-kappa} for the $N$-subject data set.  If
\begin{equation}
 \sqrt N\{L_{q,N}^2\delta_N^2+\rho_N\kappa_N\}
 \xrightarrow{p}0,
 \label{eq:main-statistical-condition}
\end{equation}
then
\[
 \sqrt N\norm{\widehat\beta_N^{\rho,\delta}-\widehat\beta_N}
 \xrightarrow{p}0,
 \qquad
 \sqrt N(\widehat\beta_N^{\rho,\delta}-\beta_0)
 \Longrightarrow\mathcal N(0,\mathcal I(\beta_0)^{-1}).
\]
In particular, if $L_{q,N}=O_p(1)$ and $\kappa_N=O_p(1)$, then for any
$\ell_N\to\infty$ one may choose
\[
 \delta_N=N^{-1/4}/\ell_N,\qquad
 \rho_N=N^{-1/2}/\ell_N.
\]
If every event risk set contains at least $cN$ subjects with probability
tending to one for some fixed $c\in(0,1)$, and if
$\ell_N=o(N^{3/4})$, maximal grouping uses
$K_N=O_p(N^{1/4}\ell_N)=o_p(N)$ shifts.
\end{theorem}

The proof in Appendix~\ref{sec:statistics} combines the uniform gradient
bound in Theorem~\ref{thm:main-compression} with local strong convexity
and Slutsky's theorem.

\paragraph{Numerical accuracy.}
Suppose that the assumptions of Corollary~\ref{cor:main-cox-rate} hold with constants
uniform in $N$: the curvature and minimum risk-set fraction are bounded away
from zero, while $X$, $M$, $L_{q,N}$, and $\kappa_N$ are bounded above.  Then
a run of $T_N$ iterations has the same first-order limit as the full empirical
Cox estimator whenever
\begin{equation}
 N\left(\frac{\log T_N}{T_N}\right)^{4/5}\longrightarrow0,
 \label{eq:main-end-to-end-statistics}
\end{equation}
by Markov's inequality and \eqref{eq:main-four-fifths}.  The precise joint
numerical and statistical conditions are recorded in
Appendix~\ref{sec:statistics}.

\section{Experiments}
\label{sec:experiments}
\label{sec:study27-experiments}

The code for the experiments is publicly available at
\url{https://github.com/elizkaveta/Cox-Regression-Analysis/}.

\subsection{Benchmark protocol}
\label{sec:study27-benchmark}
We compare our method with Batch LSE, Minibatch Cox \citep{zeng2026minibatch}, BigSurv
\citep{rhoexpTarkhanSimon2020}, and Cox-CC \citep{kvamme2019time}
on SUPPORT2~\citep{rhoexpSupportData},
NWTCO~\citep{rhoexpNwtcoData}, and three synthetic datasets
(Table~\ref{tab:study27-data}). All methods share a stratified
60/20/20 train/validation/test split on each dataset, with
preprocessing fitted on training data and Breslow handling of ties~\citep{breslow1974covariance}.

\input{tables/data.tex}

All methods initialize the coefficient vector at $\beta_0=0$
and include the ridge penalty $\lambda\|\beta\|_2^2/2$,
with $\lambda=10^{-3}$.
Our method averages mini-batches of $b$ independent stochastic
gradient samples; although the theory is stated for one oracle sample per iteration, such averaging preserves unbiasedness and the moment bounds used in the analysis.
We use the decreasing step size
$\eta_t=\eta_0(1+t/1000)^{-1/2}$
and initialize the auxiliary shifts at
$s_{k,0}=\log(1-\rho)$.
We fix $\rho=10^{-4}$ and $\delta=0.05$. The batch size $b$ and initial step size $\eta_0$ are selected by validation.
At each saved checkpoint after initialization,
optimization and validation curves for our method are evaluated at
$\bar\beta_t=t^{-1}\sum_{r=1}^{t}\beta_r$,
the arithmetic mean of the projected post-update iterates.

For each method and dataset, we evaluate twelve configurations with two
tuning seeds. We select the configuration with the smallest
mean of the two seeds' minimum saved validation Cox losses and
run it with ten optimizer seeds. All methods receive the same
per-fit vector budgets:
$U_{\rm HPO}=\max\{3\cdot10^6,100N\}$ for tuning and
$U_{\rm final}=1.5U_{\rm HPO}$ for final runs,
where $N$ is the total number of subjects.

We measure training accuracy by the regularized full-Cox gap
$\mathcal L(\beta)-\mathcal L(\beta_{\rm ref})$ to a numerical
reference. Test Cox loss and Harrell's C-index are evaluated at
validation-selected checkpoints; predictive Cox loss is unpenalized.
Appendix~\ref{sec:study27-protocol} details the baseline objectives,
vector-work accounting, preprocessing, and selection and evaluation
history.
\subsection{Optimization and prediction}
\label{sec:study27-results}

Our method achieves the lowest median terminal regularized full-Cox training gap on all five datasets at matched vector-work budgets
(Figure~\ref{fig:study27-main}). BigSurv is the closest baseline, with
median terminal gaps 1.22--26.15 times ours
(Table~\ref{tab:study27-terminal}).

\begin{figure}[t]
\centering
\includegraphics[width=\linewidth]{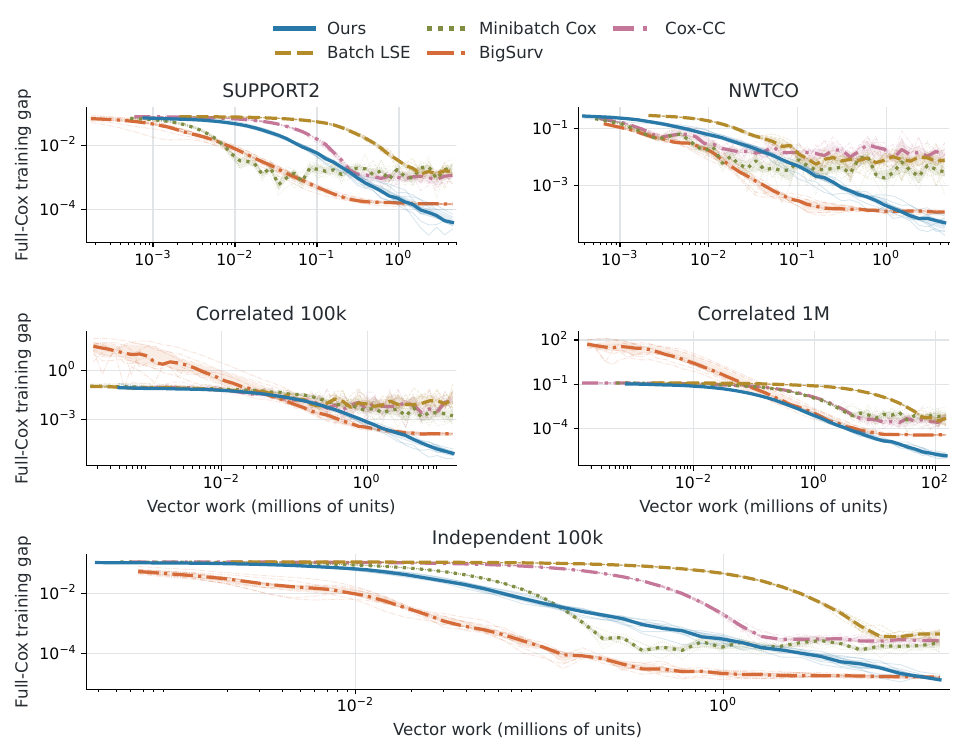}
\caption{Regularized full-Cox training gap against vector work. Each
dataset has the same cap for all five methods. Thin lines show ten
optimizer seeds, thick lines their medians, and bands their interquartile
ranges. Ours and BigSurv use running arithmetic averages of post-update
coefficient iterates. The other methods
use current coefficient iterates.}
\label{fig:study27-main}
\end{figure}

At $\delta=0.05$, grouping reduces the number of auxiliary shifts by
factors of 22.8--337.8 relative to one shift per training event
(Table~\ref{tab:study27-data}).

Additionally, to isolate the two approximation mechanisms in
Theorem~\ref{thm:main-compression}, we measure at the full-Cox reference
$\beta_{\rm ref}$
\[
 e_{\rm group}(\delta)
 =\|\nabla\widetilde{\mathcal L}_\delta(\beta_{\rm ref})
       -\nabla\mathcal L(\beta_{\rm ref})\|,
 \qquad
 e_{\rm soft}(\rho)
 =\|\nabla\mathcal L_{\rho,0.05}(\beta_{\rm ref})
       -\nabla\widetilde{\mathcal L}_{0.05}(\beta_{\rm ref})\|.
\]
Figure~\ref{fig:study27-diagnostics} examines the quadratic grouping and
linear softplus terms in~\eqref{eq:main-gradient-bound}.  Softplus error is
nearly linear in $\rho$ across all datasets.  The grouping slopes are
1.37--1.87; because the partition changes discretely with $\delta$, the
$O(\delta^2)$ bound does not imply an exact finite-grid slope of two.

\begin{figure}[t]
\centering
\includegraphics[width=\linewidth]{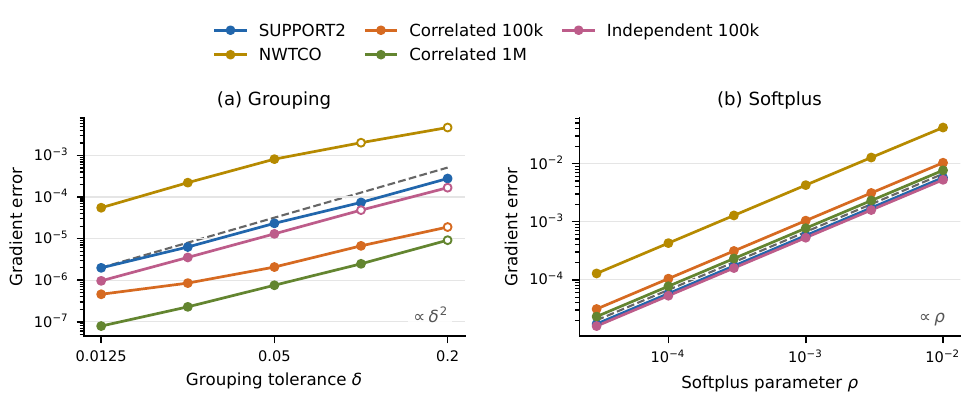}
\caption{Coefficient-gradient approximation errors at $\beta_{\rm ref}$.
The left panel isolates risk-set grouping using exact group normalizers; the
right isolates softplus profiling at $\delta=0.05$.  Filled markers satisfy
the measured pointwise conditions of Theorem~\ref{thm:main-compression};
hollow markers violate at least one applicable condition.  Dashed lines are
proportionality guides with arbitrary vertical offsets.}
\label{fig:study27-diagnostics}
\end{figure}

Our method has the lowest mean test Cox loss on NWTCO, Correlated 100k,
and Independent 100k. Minibatch Cox has the lowest mean on SUPPORT2,
and BigSurv on Correlated 1M. Full metrics and trajectories are reported
in Appendices~\ref{sec:study27-additional-results}
and~\ref{sec:study27-figures}.

Appendix~\ref{sec:study27-diagnostics} gives the parameter grids, diagnostics
at additional coefficient vectors, fitted slopes, numerical profiling
details, and pointwise checks of the sufficient conditions.

\FloatBarrier
\section{Limitations and conclusion}
\label{sec:conclusion}

Risk-set grouping converts the full Cox objective into a stochastic problem
with $K$ shared normalizers while preserving its gradient to second order in
the grouping tolerance. Under the stated approximation and curvature
conditions, this yields the $\widetilde O(T^{-4/5})$ computational rate.
Under the additional statistical conditions and parameter schedules,
the compressed estimator is first-order equivalent to the full Cox
estimator, with sublinear auxiliary state under the stated risk-set
size condition. The method is intended for stochastic-access regimes;
the experiments do not establish an end-to-end runtime advantage over
optimized cumulative-sum full-Cox solvers.

The analysis assumes a convex linear predictor and the stated
boundedness and curvature conditions. The constants $L_q$ and $\kappa$
can be large under adverse risk-weight heterogeneity. The implementation
handles equal failure times with the Breslow convention, whereas the
main statistical transfer result is stated for distinct failures.
Tied-time asymptotics, time-varying covariates, and nonlinear predictors
remain outside its scope.

At matched vector-work budgets, our method achieves the smallest median
terminal regularized full-Cox training gap on all five datasets. Together,
the theory and experiments establish a stochastic approach to full-Cox
optimization that combines explicit approximation guarantees with compact
auxiliary state.

\FloatBarrier

\subsection*{Reproducibility statement}
The appendix provides complete proofs and a detailed account of the
experimental setup, implementation, tuning, and evaluation procedures.
The accompanying source code is available at
\url{https://github.com/elizkaveta/Cox-Regression-Analysis/}.

\subsection*{AI use statement}
Generative AI assisted with language editing, literature searches, and proof
development.  All AI-assisted mathematical content was independently
verified, and the authors take full responsibility for the manuscript.

\bibliography{main}
\bibliographystyle{iclr2027/iclr2027_conference}

\clearpage
\appendix

\section{Generic grouped LogSumExp optimization}
\label{sec:generic}

This section proves Theorems~\ref{thm:main-convex}
and~\ref{thm:main-optimization}.  It first establishes the approximation,
optimum-noise, and profile-to-joint curvature lemmas, then derives the convex
averaged bound and the strongly convex last-iterate recursion.

\subsection{Softplus identities and profile regularity}

We use the assumptions and objectives of Section~\ref{sec:optimization-main}.
Write $F_k(\theta):=\log\E_k e^{L_k(X,\theta)}$ for one group's
LogSumExp term.  The scalar weight and curvature identities needed below are
\[
 w_\rho(a):=h_\rho'(a)=\frac{e^a}{1+\rho e^a},
 \qquad
 h_\rho''(a)=w_\rho(a)[1-\rho w_\rho(a)]
 \in\left(0,\frac1{4\rho}\right].
\]
For $0<\rho<1$, each scalar objective tends to $+\infty$ at both ends.  As
$s_k\to+\infty$, its leading term is $s_k$, while as
$s_k\to-\infty$, its leading term is
$(1-\rho^{-1})s_k$, both tend to $+\infty$.  Thus a scalar minimizer
exists, and it is unique because the second derivative with respect to
$s_k$ is positive.  Boundedness of the losses permits differentiation
under the expectation.  The minimizer is interior and its scalar second
derivative is nonzero, so the implicit-function theorem makes the optimal
shift differentiable in $\theta$, the envelope theorem therefore justifies
the profile gradients used below.

\begin{proposition}
\label{prop:generic-approximation}
Suppose $\rho\bar\kappa<1$.  Then, for every $\theta\in\Theta$,
\begin{equation}
 J(\theta)+\frac\rho2+\log(1-\rho\bar\kappa)
 \le J_\rho(\theta)\le J(\theta).
 \label{eq:generic-approximation}
\end{equation}
If $s_{k,\rho}(\theta)$ is the optimal shift for group $k$, then
\begin{equation}
 F_k(\theta)+\log(1-\rho\bar\kappa)
 \le s_{k,\rho}(\theta)<F_k(\theta).
 \label{eq:generic-shift-location}
\end{equation}
\end{proposition}

\begin{proof}
For each group, apply the softplus approximation and optimal-shift bounds of
\citet{gladin2025improved} with $\varphi=L_k(X,\theta)$.  Their normalized
second moment is $\E_k e^{2[L_k-F_k]}\le\bar\kappa$, so the groupwise
profile lies between $F_k+\rho/2+\log(1-\rho\bar\kappa)$ and $F_k$, and its
optimal shift lies in the interval~\eqref{eq:generic-shift-location}.
Weighting the value bounds by $p_k$ proves~\eqref{eq:generic-approximation}.
\end{proof}

\subsection{Weighted stochastic gradient method}

Write $\mathcal A=[-B-1,B]^K$ for the shift component of the feasible set
$\mathcal Z$ in Section~\ref{sec:optimization-main}.  We use the weighted
norm, stochastic gradient, projected update, and constants $C,V,D$ defined
there.  For a sampled pair $(k,X)$, the function differentiated below is
\[
 R(\theta)+Kp_k
 \left[s_k-1+h_\rho(L_k(X,\theta)-s_k)\right].
\]

\begin{lemma}
\label{lem:smoothness-noise}
Every sample function is convex and $C/\rho$-smooth in the weighted norm.
If $\rho\bar\kappa\le1/8$, $z_\rho$ minimizes $G_\rho$, and
$g(z_\rho)$ is a stochastic gradient at that point, then
\[
 \E\norm{g(z_\rho)}_{P^{-1}}^2\le V.
\]
Consequently, at every iterate,
\begin{equation}
 \E_t\norm{g_t}_{P^{-1}}^2
 \le\frac{4C}{\rho}
 [G_\rho(z_t)-G_\rho(z_\rho)]+2V.
 \label{eq:expected-smoothness}
\end{equation}
\end{lemma}

\begin{proof}
For a direction $(u,v)$, the second directional derivative of one sample
function is
\begin{align*}
 &u^\top\nabla^2R(\theta)u
 +Kp_kw\,u^\top\nabla^2L_k(X,\theta)u\\
 &\qquad
 +Kp_kh_\rho''(L_k-s_k)
 [\ip{\nabla L_k}{u}-v_k]^2.
\end{align*}
All terms are nonnegative.  Weighted Cauchy--Schwarz gives
\[
 [\ip{\nabla L_k}{u}-v_k]^2
 \le\left(G^2+\frac1{p_k}\right)
       (\norm{u}^2+p_kv_k^2).
\]
Together with $w\le1/\rho$ and
$h_\rho''\le1/(4\rho)$, this bounds the directional derivative by
$(C/\rho)\norm{(u,v)}_P^2$.

At $z_\rho$, let $w_{k,*}$ be the weight in group $k$.  The shift
first-order condition gives $\E_k w_{k,*}=1$.
Proposition~\ref{prop:generic-approximation} gives
\[
 w_{k,*}\le q_kW_k,
 \qquad
 q_k\le(1-\rho\bar\kappa)^{-1}\le\frac87,
\]
where $W_k=e^{L_k-F_k}$ and $q_k=e^{F_k-s_{k,\rho}}$.
Consequently,
\[
 \E_kw_{k,*}^2
 \le q_k^2\E_kW_k^2
 \le\frac{64}{49}\bar\kappa<2\bar\kappa\le4\bar\kappa
 \qquad\text{when }\rho\bar\kappa\le1/8.
\]
Substituting these two bounds in the stochastic gradients of
Section~\ref{sec:optimization-main} and averaging over the uniformly sampled
group gives the value of $V$ in~\eqref{eq:main-generic-constants}.

We prove \eqref{eq:expected-smoothness} directly using the globally smooth
scalar softplus. Fix $z=(\theta,s)$ in the feasible set and write
$z_*:=z_\rho=(\theta_*,s_*)$. Using the same sampled pair $(k,X)$
for the current and optimal weights, set
\[
 a=L_k(X,\theta)-s_k,\qquad
 a_*=L_k(X,\theta_*)-s_{k,*},\qquad
 w=h_\rho'(a),\qquad w_*=h_\rho'(a_*).
\]
For a differentiable function $F$, write
$D_F(x,y):=F(x)-F(y)-\ip{\nabla F(y)}{x-y}$.
In particular, with $h=h_\rho$,
\[
 D_h(a,a_*)=h_\rho(a)-h_\rho(a_*)-w_*(a-a_*).
\]
Since $h_\rho$ is convex and $(4\rho)^{-1}$-smooth on all of
$\mathbb R$, scalar co-coercivity gives
\[
 (w-w_*)^2\le\frac{D_h(a,a_*)}{2\rho}.
\]

Decompose the sampled Euclidean gradient as $g(z)=b+r$, where
\begin{align*}
 b&=\bigl(
   \nabla R(\theta)+Kp_kw_*\nabla L_k(X,\theta),\,
   Kp_k(1-w_*)e_k
 \bigr),\\
 r&=Kp_k(w-w_*)\bigl(\nabla L_k(X,\theta),-e_k\bigr),
\end{align*}
and $e_k$ is the $k$th coordinate vector in $\mathbb R^K$.
The previously established identities
$\E_k w_*=1$ and $\E_k w_*^2\le4\bar\kappa$, together with the
uniform gradient bounds and $\sum_kp_k^2\le1$, imply
\begin{align*}
 \E\norm b_{P^{-1}}^2
 &\le 2G_R^2+
  2KG^2\sum_kp_k^2\E_k w_*^2+
  K\sum_kp_k\E_k(1-w_*)^2\\
 &\le 2G_R^2+8K\bar\kappa G^2+4K\bar\kappa-K
 \le V.
\end{align*}
Moreover, since $p_k\le1$,
\begin{align*}
 \E\norm r_{P^{-1}}^2
 &\le K(G^2+1)\sum_kp_k\E_k(w-w_*)^2\\
 &\le\frac{K(G^2+1)}{2\rho}
       \sum_kp_k\E_kD_h(a,a_*).
\end{align*}

The Bregman divergence of the full objective satisfies the exact identity
\[
 D_{G_\rho}(z,z_*)
 =D_R(\theta,\theta_*)+
   \sum_kp_k\E_k\bigl[
     D_h(a,a_*)+w_*D_{L_k(X,\cdot)}(\theta,\theta_*)
   \bigr].
\]
Here the same sampled $X$ is used in both arguments of
$D_{L_k(X,\cdot)}$.
Convexity of $R$ and every $L_k(X,\cdot)$, and $w_*>0$, give
\[
 D_{G_\rho}(z,z_*)\ge\sum_kp_k\E_kD_h(a,a_*).
\]
Consequently,
\begin{align*}
 \E\norm{g(z)}_{P^{-1}}^2
 &\le2\E\norm b_{P^{-1}}^2+2\E\norm r_{P^{-1}}^2\\
 &\le2V+\frac{K(G^2+1)}{\rho}D_{G_\rho}(z,z_*)\\
 &\le2V+\frac{4C}{\rho}
       [G_\rho(z)-G_\rho(z_*)].
\end{align*}
The last inequality uses $4C\ge K(G^2+1)$ and constrained
first-order optimality,
$\ip{\nabla G_\rho(z_*)}{z-z_*}\ge0$.
Conditioning on the history through $z_t$ and taking a fresh sample
gives \eqref{eq:expected-smoothness}.
\end{proof}

\subsection{Curvature of the joint surrogate}

The pointwise second derivative of softplus can be very small.  The next
scalar inequality instead measures curvature relative to the optimal shift.

\begin{lemma}
\label{lem:softplus-secant}
Fix $a_*$ and let
\[
 w_*:=w_\rho(a_*),
 \qquad
 m_*:=w_*[1-\rho w_*].
\]
Then, for every $a\in\R$,
\begin{equation}
 [w_\rho(a)-w_*](a-a_*)
 \ge m_*\frac{(a-a_*)^2}{1+|a-a_*|}.
 \label{eq:softplus-secant}
\end{equation}
\end{lemma}

\begin{proof}
Put $d=a-a_*$.  If $d\ge0$, direct substitution gives
\[
 w_\rho(a_*+d)-w_*
 =\frac{m_*(e^d-1)}{1+\rho w_*(e^d-1)}
 \ge m_*(1-e^{-d}).
\]
If $d=-x\le0$, the same calculation gives
\[
 w_*-w_\rho(a_*-x)
 =\frac{m_*(1-e^{-x})}{1-\rho w_*(1-e^{-x})}
 \ge m_*(1-e^{-x}).
\]
The weight is increasing, and
$1-e^{-x}\ge x/(1+x)$ for $x\ge0$.  Multiplying by $|d|$ proves
\eqref{eq:softplus-secant}.
\end{proof}

At the surrogate minimizer, the average scalar curvature is not small.
Indeed, $\E_kw_{k,*}=1$ and $\E_kw_{k,*}^2\le4\bar\kappa$, so
\begin{equation}
 \E_k\{w_{k,*}[1-\rho w_{k,*}]\}
 =1-\rho\E_kw_{k,*}^2\ge\frac12.
 \label{eq:average-shift-curvature}
\end{equation}

We now state the only additional assumption needed for the fast generic
rate.  The profiled objective $J_\rho$ is called $\gamma$-strongly convex if
\[
 J_\rho(\theta')\ge J_\rho(\theta)
 +\ip{\nabla J_\rho(\theta)}{\theta'-\theta}
 +\frac\gamma2\norm{\theta'-\theta}^2.
\]
When $J_\rho$ is twice differentiable, a sufficient condition is
$\nabla^2J_\rho(\theta)\succeq\gamma I$ on $\Theta$.  The optimal shift is unique and
interior, and its Hessian block is positive definite.  Differentiating its
first-order condition and then the envelope identity shows that the profile
Hessian is the Schur complement
\begin{equation}
 \nabla^2J_\rho
 =G_{\theta\theta}
  -G_{\theta s}G_{ss}^{-1}G_{s\theta},
 \label{eq:profile-hessian}
\end{equation}
where all blocks are evaluated at the optimal shifts.  Thus profile
curvature concerns the parameter after the auxiliary variables have adjusted
optimally.

\begin{lemma}
\label{lem:generic-rsi}
Suppose $\rho\bar\kappa\le1/8$ and $J_\rho$ is
$\gamma$-strongly convex.  Define
\begin{equation}
 \mu
 :=\left[\frac{2(1+2G^2)}\gamma+16B+8\right]^{-1}.
 \label{eq:generic-mu}
\end{equation}
Then, for every $z\in\Theta\times\mathcal A$,
\begin{equation}
 \ip{\nabla G_\rho(z)-\nabla G_\rho(z_\rho)}{z-z_\rho}
 \ge\mu\norm{z-z_\rho}_P^2.
 \label{eq:generic-rsi}
\end{equation}
\end{lemma}

\begin{proof}
Write $u=\theta-\theta_\rho$ and
$v_k=s_k-s_{k,\rho}$.  For an observation in group $k$, put
\[
 d_k=L_k(X,\theta)-L_k(X,\theta_\rho)-v_k,
 \qquad
 m_k=w_{k,*}(1-\rho w_{k,*}).
\]
Indeed, $|L_k(X,\theta)-L_k(X,\theta_\rho)|\le2B$ and
$|v_k|\le2B+1$, so $|d_k|\le4B+1$.  Convexity of each loss and
Lemma~\ref{lem:softplus-secant} give the needed samplewise lower bound as
follows.  Write
\[
 a=L_k(X,\theta)-s_k,
 \qquad
 a_*=L_k(X,\theta_\rho)-s_{k,\rho}.
\]
The two convexity inequalities for $L_k$ imply
\begin{align*}
 &\ip{w_\rho(a)\nabla L_k(X,\theta)
       -w_\rho(a_*)\nabla L_k(X,\theta_\rho)}{u}
 -[w_\rho(a)-w_\rho(a_*)]v_k\\
 &\qquad\ge
 [w_\rho(a)-w_\rho(a_*)]
 [L_k(X,\theta)-L_k(X,\theta_\rho)-v_k]\\
 &\qquad=[w_\rho(a)-w_\rho(a_*)]d_k.
\end{align*}
Lemma~\ref{lem:softplus-secant} and $|d_k|\le4B+1$ bound this by
$m_kd_k^2/(4B+2)$.  After taking expectations and weighting the groups, the
softplus contribution to the gradient secant is therefore at least
\[
 \frac{A}{4B+2},
 \qquad
 A:=\sum_{k=1}^Kp_k\E_k(m_kd_k^2).
\]
The regularizer is convex, so its contribution is nonnegative.  Hence, if
$S$ denotes the left-hand side of~\eqref{eq:generic-rsi},
\begin{equation}
 S\ge\frac{A}{4B+2}.
 \label{eq:rsi-A-bound}
\end{equation}

Next,
$v_k=L_k(X,\theta)-L_k(X,\theta_\rho)-d_k$ and
$|L_k(X,\theta)-L_k(X,\theta_\rho)|\le G\norm u$.  Multiply
$(r+s)^2\le2r^2+2s^2$ by $m_k$ and take expectations to obtain
\[
 v_k^2\E_km_k
 \le2G^2\norm u^2\E_km_k+2\E_k(m_kd_k^2).
\]
Equation~\eqref{eq:average-shift-curvature} gives
$\E_km_k\ge1/2$ (and trivially $\E_km_k\le1$), so division by
$\E_km_k$ and averaging with the $p_k$ gives
\[
 \sum_{k=1}^Kp_kv_k^2
 \le4A+2G^2\norm u^2,
\]
and therefore
\begin{equation}
 \norm{z-z_\rho}_P^2
 \le(1+2G^2)\norm u^2+4A.
 \label{eq:rsi-norm-bound}
\end{equation}

The symmetric Bregman divergence of $G_\rho$ is the sum of its two
one-sided Bregman divergences and equals $S$.  Convexity makes both summands
nonnegative, so $S$ is at least either one.  At the optimal shifts,
$\nabla_s G_\rho(z_\rho)=0$ and
$\nabla_\theta G_\rho(z_\rho)=\nabla J_\rho(\theta_\rho)$.  Since
$G_\rho(\theta,s)\ge J_\rho(\theta)$ and equality holds at $z_\rho$,
the one-sided Bregman divergence satisfies
\[
 G_\rho(z)-G_\rho(z_\rho)
 -\ip{\nabla G_\rho(z_\rho)}{z-z_\rho}
 \ge
 J_\rho(\theta)-J_\rho(\theta_\rho)
 -\ip{\nabla J_\rho(\theta_\rho)}{u}.
\]
Strong convexity of $J_\rho$ now gives
\begin{equation}
 S\ge\frac\gamma2\norm u^2.
 \label{eq:rsi-u-bound}
\end{equation}
Equations~\eqref{eq:rsi-A-bound} and~\eqref{eq:rsi-u-bound} imply
\[
 \norm{z-z_\rho}_P^2
 \le\left[\frac{2(1+2G^2)}\gamma+4(4B+2)\right]S.
\]
This is~\eqref{eq:generic-rsi}.
\end{proof}

\begin{remark}
If $R$ is $\lambda$-strongly convex, then
$G_\rho(\theta,s)-\lambda\norm\theta^2/2$ is jointly convex.  Partial
minimization over $s$ preserves convexity, so $J_\rho$ remains
$\lambda$-strongly convex and one may take $\gamma=\lambda$.  The Cox
application below instead obtains $\gamma$ from the partial likelihood
itself, and therefore permits $\lambda=0$.
\end{remark}

\subsection{Generic convergence bounds}

\begin{proof}[Proof of Theorem~\ref{thm:main-convex}]
Write $\bar z_T=T^{-1}\sum_{t=1}^Tz_t$.
Nonexpansiveness of the weighted projection, unbiasedness, convexity, and
Lemma~\ref{lem:smoothness-noise} give
\begin{align*}
 \E_t\norm{z_{t+1}-z_\rho}_P^2
 \le{}&\norm{z_t-z_\rho}_P^2
 -\left(2\eta_T-\frac{4C\eta_T^2}{\rho_T}\right)
   [G_\rho(z_t)-G_\rho(z_\rho)]
 +2\eta_T^2V.
\end{align*}
The chosen step size makes the coefficient in parentheses equal to
$\eta_T$.  Sum over $t$, use $\norm{z_1-z_\rho}_P^2\le D^2$, and apply
Jensen's inequality to the averaged iterate.  This yields
\[
 \E[G_\rho(\bar z_T)-G_\rho(z_\rho)]
 \le\frac{4CD^2}{\rho_TT}+\frac{\rho_TV}{2C}.
\]
Proposition~\ref{prop:generic-approximation} transfers this inequality from
$G_\rho$ to $J$.  Since $\rho_T\bar\kappa\le1/8$, the approximation error
is at most $-\rho_T/2-\log(1-\rho_T\bar\kappa)
\le8\rho_T\bar\kappa/7$.
\end{proof}

The fast result is most useful in distance form, because applications can
compare the surrogate minimizer with the desired target through a
problem-specific gradient bound.

\begin{proof}[Proof of Theorem~\ref{thm:main-optimization}]
Let $z_\rho$ minimize $G_\rho$.  As in the convex proof,
nonexpansiveness and unbiasedness give
\[
 \E_t\norm{z_{t+1}-z_\rho}_P^2
 \le\norm{z_t-z_\rho}_P^2
 -2\eta\ip{\nabla G_\rho(z_t)}{z_t-z_\rho}
 +\eta^2\E_t\norm{g_t}_{P^{-1}}^2.
\]
The gradient inner product is at least the surrogate gap by convexity and at
least $\mu\norm{z_t-z_\rho}_P^2$ by
Lemma~\ref{lem:generic-rsi} and first-order optimality.  It is therefore at
least half the sum of these two lower bounds.  Substitute this fact and
\eqref{eq:expected-smoothness}.  Since $\eta\le\rho/(4C)$, the coefficient
of the nonnegative surrogate gap is nonpositive and may be dropped, giving
\[
 \E_t\norm{z_{t+1}-z_\rho}_P^2
 \le(1-\mu\eta)\norm{z_t-z_\rho}_P^2+2\eta^2V.
\]
Because $0\le1-\mu\eta<1$, iteration of this recursion and the geometric
sum give
\[
 \E\norm{z_{T+1}-z_\rho}_P^2
 \le(1-\mu\eta)^TD^2
   +2\eta^2V\sum_{j=0}^{T-1}(1-\mu\eta)^j
 \le e^{-\mu\eta T}D^2+\frac{2\eta V}{\mu},
\]
which proves~\eqref{eq:main-parameterized-rate}.  For the displayed
schedule, $\eta_T=\rho_T/(4C)$,
$\mu\eta_TT=2\log T$, and $\eta_T\le1/\mu$ for $T\ge2$.  Substitution proves
\eqref{eq:main-optimized-generic-rate} after dropping the shift coordinates.
This last step is the standard constant-step strongly convex SGD recursion
\citep{moulines2011nonasymptotic}.
\end{proof}

\begin{proof}[Proof of Corollary~\ref{cor:main-generic-objective}]
Let $\theta^\star$ minimize $J$.  By the value approximation
\eqref{eq:generic-approximation}, optimality and the two sides of that bound
give
\[
 J_\rho(\theta^\star)-J_\rho(\theta_\rho)
 \le J(\theta^\star)-J(\theta_\rho)-\rho/2-\log(1-\rho\bar\kappa)
 \le-\rho/2-\log(1-\rho\bar\kappa).
\]
Strong convexity of $J_\rho$ therefore gives
\[
 \norm{\theta_\rho-\theta^\star}^2
  \le\frac{2}{\gamma}[-\rho/2-\log(1-\rho\bar\kappa)]
 \le\frac{16\rho\bar\kappa}{7\gamma}.
\]
The smoothness assumptions at the start of this section imply that $J$ is
$L_J$-smooth.  Since $\theta^\star$ is stationary, smoothness, the preceding
display, \eqref{eq:main-optimized-generic-rate}, and
$\norm{a+b}^2\le2\norm a^2+2\norm b^2$ give
\eqref{eq:main-generic-objective-rate}.  Finally,
$\rho_T=8C\log T/(\mu T)$ yields the stated order.
\end{proof}

Stationarity is needed because the smoothness upper bound contains a linear
term at a constrained boundary.  In Cox regression we use a sharper gradient
comparison, its softplus contribution is quadratic after conversion to
objective error.

\section{Cox regression}
\label{sec:cox}

This section proves the uniform comparison in
Theorem~\ref{thm:main-compression} and specializes the generic recursion to
obtain Corollary~\ref{cor:main-cox-rate}, with all constants retained.

\subsection{Notation and data constants}

We use the data, grouping rule, and objectives of
Section~\ref{sec:cox-main}: the exact $\mathcal L$, grouped
$\widetilde{\mathcal L}_\delta$, joint $G_{\rho,\delta}$, and profiled
$\mathcal L_{\rho,\delta}$.  The optional ridge coefficient $\lambda\ge0$
is the same in all four objectives.  Throughout this section,
$\norm{x_j}\le X$ and $|\beta^\top x_j|\le M$ for every subject $j$ and
$\beta\in\mathcal B$.

For $i\in\mathcal E$, let $Q_i^\beta$ be the Cox distribution on
$\mathcal R_i$:
\[
 Q_i^\beta(j)
 :=\frac{e^{\beta^\top x_j}}
         {\sum_{\ell\in\mathcal R_i}e^{\beta^\top x_\ell}}.
\]
The constants $\kappa$ and $L_q$ are defined in~\eqref{eq:main-kappa}
and~\eqref{eq:main-Lq}.  Both are independent of the grouping.

The suprema over $\mathcal B$ are needed for a global optimization theorem.
For a local statistical comparison they may instead be taken over a fixed
neighborhood of the target.

\subsection{Grouping bounds}

Recall $r_\delta=\delta/(1+\delta)$ from
Section~\ref{sec:cox-main}.  By the grouping rule, at most this fraction
of the largest risk set disappears inside one group.

\begin{lemma}
\label{lem:nearby-risk-sets}
Suppose
\begin{equation}
 r_\delta\le1-q,
 \qquad
 L_qr_\delta\le\frac12.
 \label{eq:grouping-smallness}
\end{equation}
For any two risk sets $\mathcal R_v\subset\mathcal R_u$ in the same group,
\begin{align}
 |a_u(\beta)-a_v(\beta)|&\le2L_qr_\delta,
 \label{eq:normalizer-oscillation}\\
 \norm{\nabla a_u(\beta)-\nabla a_v(\beta)}&\le2XL_qr_\delta,
 \label{eq:gradient-oscillation}\\
 \norm{\nabla^2a_u(\beta)-\nabla^2a_v(\beta)}_{\rm op}
 &\le5X^2L_qr_\delta.
 \label{eq:hessian-oscillation}
\end{align}
\end{lemma}

\begin{proof}
Let $A=\mathcal R_u\setminus\mathcal R_v$ and define its cardinality and
Cox-mass fractions by
\[
 r:=\frac{|A|}{|\mathcal R_u|},
 \qquad
 \alpha:=Q_u^\beta(A).
\]
The first condition in~\eqref{eq:grouping-smallness} makes the pair
admissible in~\eqref{eq:main-Lq}.  Therefore
\begin{equation}
 r\le r_\delta,
 \qquad
 \alpha\le L_qr\le L_qr_\delta.
 \label{eq:cardinality-and-mass}
\end{equation}
Removing $A$ changes the average risk weight according to the exact identity
\begin{equation}
 a_u-a_v=\log\frac{1-r}{1-\alpha}.
 \label{eq:normalizer-identity}
\end{equation}
Both $r$ and $\alpha$ lie in $[0,L_qr_\delta]$.  Hence
\[
 |a_u-a_v|
 \le-\log(1-L_qr_\delta)
 \le2L_qr_\delta,
\]
which proves~\eqref{eq:normalizer-oscillation}.

Let $Q_A$ be the Cox distribution conditional on $A$.  The outer Cox law is
the mixture
\[
 Q_u^\beta=(1-\alpha)Q_v^\beta+\alpha Q_A.
\]
Its mean therefore differs from that of $Q_v^\beta$ by at most $2X\alpha$,
which proves~\eqref{eq:gradient-oscillation}.

Write $H_u,H_v,H_A$ for the corresponding covariance matrices and
$g_v,g_A$ for the two conditional means.  The covariance of a mixture is
\[
 H_u=(1-\alpha)H_v+\alpha H_A
 +\alpha(1-\alpha)(g_A-g_v)(g_A-g_v)^\top.
\]
All covariance matrices are between $0$ and $X^2I$, while
$\norm{g_A-g_v}\le2X$.  Consequently
\[
 \norm{H_u-H_v}_{\rm op}
 \le \alpha X^2+4\alpha X^2
 \le5X^2L_qr_\delta.
\]
This proves the last claim.
\end{proof}

\begin{proposition}
\label{prop:grouping-gradient-curvature}
Under~\eqref{eq:grouping-smallness}, for every $\beta\in\mathcal B$,
\begin{align}
 \norm{\nabla\widetilde{\mathcal L}_\delta(\beta)
       -\nabla\mathcal L(\beta)}
 &\le XL_q^2r_\delta^2,
 \label{eq:grouping-gradient}\\
 \nabla^2\widetilde{\mathcal L}_\delta(\beta)
 &\succeq
 \nabla^2\mathcal L(\beta)-\frac52X^2L_q^2r_\delta^2I.
 \label{eq:grouping-curvature}
\end{align}
\end{proposition}

\begin{proof}
Fix a group and suppress $\beta$ in $a_i(\beta)$.  Write
$g_i=\nabla a_i(\beta)$ and $H_i=\nabla^2a_i(\beta)$.  Differentiating
\eqref{eq:main-group-normalizer} gives
\begin{align}
 \nabla\Phi_k&=\sum_{i\in\mathcal I_k}\pi_i g_i,
 \label{eq:group-gradient-formula}\\
 \nabla^2\Phi_k&=\sum_{i\in\mathcal I_k}\pi_iH_i
                  +\operatorname{Cov}_{\pi}(g_i),
 \label{eq:group-hessian-formula}
\end{align}
where $\pi_i=e^{a_i}/\sum_{r\in\mathcal I_k}e^{a_r}$.  The exact Cox group
uses the uniform average of the $g_i$ and $H_i$.

By Lemma~\ref{lem:nearby-risk-sets}, the oscillation of the $a_i$ is at most
$2L_qr_\delta$, the diameter of the $g_i$ is at most
$2XL_qr_\delta$, and the diameter of the $H_i$ is at most
$5X^2L_qr_\delta$.  To compare the group weights, interpolate from the
uniform law $u$ to $\pi$ via
$\pi_{t,i}=e^{ta_i}/\sum_{r\in\mathcal I_k}e^{ta_r}$, $0\le t\le1$.
For any scalar array $f_i$, write
$\osc(f)=\max_i f_i-\min_i f_i$.  Then
$\frac{d}{dt}\E_{\pi_t}f=\operatorname{Cov}_{\pi_t}(f,a)$.
Cauchy--Schwarz and the variance bound
$\operatorname{Var}(f)\le\osc(f)^2/4$ therefore give
\begin{equation}
 |\E_\pi f-\E_u f|
 \le\frac14\osc(f)\osc(a).
 \label{eq:group-reweighting}
\end{equation}
Apply this to $f_i=v^\top g_i$ and take the supremum over unit vectors
$v$.  The groupwise gradient difference is at most
\[
 \left\|\sum_i\pi_i g_i-\frac1{m_k}\sum_i g_i\right\|
 \le(2XL_qr_\delta)\frac{2L_qr_\delta}{4}
 =XL_q^2r_\delta^2.
\]
For the Hessian, apply~\eqref{eq:group-reweighting} to
$f_i=v^\top H_iv$ for each unit $v$.  The covariance term in
\eqref{eq:group-hessian-formula} is positive semidefinite, so
\[
 \nabla^2\Phi_k
 \succeq\frac1{m_k}\sum_iH_i
 - (5X^2L_qr_\delta)\frac{2L_qr_\delta}{4}I.
\]
Weighting the groupwise bounds by $p_k$ proves the proposition.
\end{proof}

The same oscillation bound removes the group dependence from the softplus
moment.

\begin{lemma}
\label{lem:group-kappa-envelope}
Under~\eqref{eq:grouping-smallness}, every group satisfies
\begin{equation}
 \sup_{\beta\in\mathcal B}
 \E_k\exp\!\left(2[\beta^\top x_J-\Phi_k(\beta)]\right)
 \le\frac98\kappa.
 \label{eq:group-kappa-envelope}
\end{equation}
\end{lemma}

\begin{proof}
Put
\[
 A_i=e^{a_i(\beta)},
 \qquad
 B_i=\frac1{n_i}\sum_{j\in\mathcal R_i}e^{2\beta^\top x_j}.
\]
Definition~\eqref{eq:main-kappa} gives
$B_i\le\kappa A_i^2$.  Consequently the group moment is at most
\[
 \kappa
 \frac{m_k^{-1}\sum_iA_i^2}
      {(m_k^{-1}\sum_iA_i)^2}.
\]
If positive numbers $y$ lie in $[a,b]$, then
$(y-a)(b-y)\ge0$ implies $y^2\le(a+b)y-ab$.  Writing
$\overline y=\E y$ and maximizing over $\overline y\in[a,b]$ gives
\[
 \frac{\E y^2}{(\E y)^2}
 \le\frac{a+b}{\overline y}-\frac{ab}{\overline y^2}
 \le\frac{(a+b)^2}{4ab}
\]
the last maximum occurs at $\overline y=2ab/(a+b)$.
For any pair in the group, identity~\eqref{eq:normalizer-identity} shows that
both $e^{a_u-a_v}$ and $e^{a_v-a_u}$ are at most
$(1-L_qr_\delta)^{-1}$.  Hence the largest and smallest group values obey the
sharper ratio bound
$b/a=e^{\osc a}\le(1-L_qr_\delta)^{-1}\le2$.  The last display is therefore at most
$(1+2)^2/(4\cdot2)=9/8$.
\end{proof}

\subsection{Softplus gradient and curvature}

The next lemma is stated for a single group distribution.  It will be
applied with the fixed envelope from
Lemma~\ref{lem:group-kappa-envelope}.  Its $\bar\kappa$ is the group-moment
envelope of Section~\ref{sec:optimization-main}, bounded here by
$9\kappa/8$.

\begin{lemma}
\label{lem:softplus-gradient-curvature}
Let
\[
 \Phi(\beta)=\log\E e^{\beta^\top X}
\]
for a distribution supported on $\norm X\le X_0$, and suppose
\[
 \E e^{2[\beta^\top X-\Phi(\beta)]}\le\bar\kappa.
\]
Let $\Phi_\rho$ be its profiled softplus representation.  If
$\rho\bar\kappa\le1/8$, then
\begin{align}
 \norm{\nabla \Phi_\rho(\beta)-\nabla \Phi(\beta)}
 &\le3X_0\rho\bar\kappa,
 \label{eq:one-group-softplus-gradient}\\
 \nabla^2\Phi_\rho(\beta)
 &\succeq\nabla^2\Phi(\beta)-16X_0^2\rho\bar\kappa I.
 \label{eq:one-group-softplus-hessian}
\end{align}
\end{lemma}

\begin{proof}
Put $W=e^{\beta^\top X-\Phi(\beta)}$, so $\E W=1$ and
$\E W^2\le\bar\kappa$.  At the optimal shift, write
\[
 c=e^{\Phi(\beta)-s_\rho(\beta)},
 \qquad
 w=\frac{cW}{1+\rho cW}.
\]
The shift first-order condition is $\E w=1$, and
Proposition~\ref{prop:generic-approximation} gives
\[
 1<c\le\frac1{1-\rho\bar\kappa}.
\]
Therefore
\begin{align*}
 \E|w-W|
 &\le(c-1)\E W+\rho c\E W^2\\
 &\le\frac{2\rho\bar\kappa}{1-\rho\bar\kappa}
 \le3\rho\bar\kappa.
\end{align*}
Since the two gradients are $\E(wX)$ and $\E(WX)$, this proves
\eqref{eq:one-group-softplus-gradient}.

For the Hessian, let
\[
 m=w(1-\rho w).
\]
The exact and profiled Hessians have the weighted least-squares forms
\begin{align}
 v^\top\nabla^2\Phi(\beta)v
 &=\min_{r\in\R}\E[W(v^\top X-r)^2],
 \label{eq:exact-weighted-variance}\\
 v^\top\nabla^2\Phi_\rho(\beta)v
 &=\min_{r\in\R}\E[m(v^\top X-r)^2].
 \label{eq:soft-weighted-variance}
\end{align}
For the first identity, $\E W=1$ and
$\nabla^2\Phi=\E(WXX^\top)-\E(WX)\E(WX)^\top$.  For the second, the joint
softplus Hessian blocks are
\[
 G_{\beta\beta}=\E(mXX^\top),\qquad
 G_{\beta s}=-\E(mX),\qquad
 G_{ss}=\E m.
\]
Taking their Schur complement gives
$\E(mXX^\top)-\E(mX)\E(mX)^\top/\E m$, which is exactly the minimum over
$r$ in~\eqref{eq:soft-weighted-variance}.

Moreover,
\[
 \E w^2\le c^2\E W^2
 \le\frac{\bar\kappa}{(1-\rho\bar\kappa)^2},
\]
and hence, with $a=\rho\bar\kappa\le1/8$,
\[
 \E|m-W|
 \le\E|w-W|+\rho\E w^2
 \le\frac{2a}{1-a}+\frac{a}{(1-a)^2}
 =\frac{a(3-2a)}{(1-a)^2}
 \le4a=4\rho\bar\kappa.
\]
Let $r_m$ minimize~\eqref{eq:soft-weighted-variance}.  It is a weighted
mean of $v^\top X$ and therefore lies between the smallest and largest
values of $v^\top X$.  Thus
$|v^\top X-r_m|\le2X_0\norm v$.  Evaluating both weighted variances at
$r_m$ gives
\begin{align*}
 v^\top\nabla^2\Phi_\rho v
 &=\E[m(v^\top X-r_m)^2]\\
 &\ge\min_r\E[W(v^\top X-r)^2]
      -4X_0^2\norm v^2\E|m-W|\\
  &\ge v^\top\nabla^2\Phi\,v-16X_0^2\rho\bar\kappa\norm v^2.
\end{align*}
This proves~\eqref{eq:one-group-softplus-hessian}.
\end{proof}

Combining the preceding results gives one finite-sample comparison that will
be used twice: first for computation and later for statistics.

\begin{proposition}
\label{prop:cox-total-comparison}
Suppose~\eqref{eq:grouping-smallness} holds and
\begin{equation}
 \rho\kappa\le\frac19.
 \label{eq:cox-rho-small}
\end{equation}
Then, uniformly over $\beta\in\mathcal B$,
\begin{align}
 \norm{\nabla\mathcal L_{\rho,\delta}(\beta)
       -\nabla\mathcal L(\beta)}
 &\le X L_q^2r_\delta^2+4X\rho\kappa,
 \label{eq:cox-total-gradient}\\
 \nabla^2\mathcal L_{\rho,\delta}(\beta)
 &\succeq\nabla^2\mathcal L(\beta)
 -\left(\frac52X^2L_q^2r_\delta^2
        +20X^2\rho\kappa\right)I.
 \label{eq:cox-total-hessian}
\end{align}
\end{proposition}

\begin{proof}
Lemma~\ref{lem:group-kappa-envelope} bounds each group moment by
$9\kappa/8$.  Condition~\eqref{eq:cox-rho-small} therefore makes
$\rho(9\kappa/8)\le1/8$ in every group.  Lemma
\ref{lem:softplus-gradient-curvature}, followed by averaging with the $p_k$,
gives a softplus gradient error smaller than $4X\rho\kappa$ and a
Hessian loss smaller than $20X^2\rho\kappa$.  Add the grouping
bounds in Proposition~\ref{prop:grouping-gradient-curvature}.
\end{proof}

\subsection{Curvature transfer and deterministic minimizer comparison}

We now impose the main Cox assumption:
\begin{equation}
 \nabla^2\mathcal L(\beta)\succeq\nu I
 \qquad\text{for every }\beta\in\mathcal B,
 \label{eq:exact-cox-curvature}
\end{equation}
for some $\nu>0$.  When $\lambda=0$, this is an identifiability and
design condition on the exact partial likelihood, not curvature supplied by
a penalty.

\begin{theorem}
\label{thm:cox-transfer}
Assume~\eqref{eq:exact-cox-curvature},
\eqref{eq:grouping-smallness}, and~\eqref{eq:cox-rho-small}.  Let
$\beta^\star$ minimize $\mathcal L$ and let $\beta_{\rho,\delta}$ minimize
$\mathcal L_{\rho,\delta}$ on $\mathcal B$.  Then
\begin{equation}
 \norm{\beta_{\rho,\delta}-\beta^\star}
 \le
 \frac X\nu\left(L_q^2r_\delta^2+4\rho\kappa\right).
 \label{eq:cox-minimizer-distance}
\end{equation}
If, in addition,
\begin{equation}
 \frac52X^2L_q^2r_\delta^2+20X^2\rho\kappa
 \le\frac\nu2,
 \label{eq:cox-curvature-budget}
\end{equation}
then
\begin{equation}
 \nabla^2\mathcal L_{\rho,\delta}(\beta)
 \succeq\frac\nu2I
 \qquad(\beta\in\mathcal B).
 \label{eq:profile-cox-curvature}
\end{equation}
\end{theorem}

\begin{proof}
Under the additional curvature budget, the curvature conclusion follows
immediately from \eqref{eq:cox-total-hessian}.  For the distance bound, which
does not use that budget, put
$h=\beta_{\rho,\delta}-\beta^\star$.  Strong convexity of the exact
objective gives
\[
 \ip{\nabla\mathcal L(\beta_{\rho,\delta})
      -\nabla\mathcal L(\beta^\star)}{h}
 \ge\nu\norm h^2.
\]
The variational inequalities for the two constrained minimizers give
\[
 \ip{\nabla\mathcal L(\beta^\star)}{h}\ge0,
 \qquad
 \ip{\nabla\mathcal L_{\rho,\delta}
      (\beta_{\rho,\delta})}{h}\le0.
\]
Insert and subtract
$\nabla\mathcal L_{\rho,\delta}(\beta_{\rho,\delta})$ in the
left side of the strong-convexity display.  The two variational-inequality
terms are nonpositive after this rearrangement, so
\[
 \nu\norm h^2
 \le
 \ip{\nabla\mathcal L(\beta_{\rho,\delta})
       -\nabla\mathcal L_{\rho,\delta}
        (\beta_{\rho,\delta})}{h}
 \le
 \norm{\nabla\mathcal L(\beta_{\rho,\delta})
       -\nabla\mathcal L_{\rho,\delta}
        (\beta_{\rho,\delta})}\norm h.
\]
Now use~\eqref{eq:cox-total-gradient}. If $h=0$ there is nothing to prove.
\end{proof}

\subsection{End-to-end computational rate}

For a declared horizon $T$, choose the grouping tolerance $\delta_T$, build
the groups once, and keep them fixed throughout the run.  Let $K_T$ be their
number and define
\begin{align}
 A_{K_T}&:=\lambda+\frac{K_T}{4}(X^2+1),
 \label{eq:cox-C}\\
 V_T&:=2\lambda^2B_0^2
       +\frac92K_T\kappa(5X^2+1),
 \label{eq:cox-V}\\
 D^2&:=\operatorname{diam}(\mathcal B)^2+(2M+1)^2,
 \label{eq:cox-D}
\end{align}
where $B_0:=\sup_{\beta\in\mathcal B}\norm\beta$.  These are the generic
constants~\eqref{eq:main-generic-constants}--\eqref{eq:main-generic-diameter},
using the fixed group moment envelope $9\kappa/8$.

Under the curvature budget~\eqref{eq:cox-curvature-budget}, the profiled
Cox surrogate has curvature at least $\nu/2$.  Taking $\gamma=\nu/2$
in~\eqref{eq:main-generic-mu} gives the $T$-independent restricted-secant
constant $\mu$ stated in Corollary~\ref{cor:main-cox-rate}.
This constant has only polynomial dependence on $M$.  The separate
intrinsic factors $L_q$ and $\kappa$ may still
be large when risk weights are highly heterogeneous.

The specialized stochastic step samples a group uniformly, then an event
uniformly from that group, and finally a subject uniformly from the event's
risk set.  Its expected objective has
\[
 R(\beta)=\frac\lambda2\norm\beta^2
          -\beta^\top\bar x_{\mathcal E},
 \qquad
 L(\beta,x)=\beta^\top x.
\]
The weighted projection and shift box
$[-M-1,M]^{K_T}$ are as in~\eqref{eq:main-generic-update}.
All expectations in this subsection are conditional on the data and are over
the random draws made by the algorithm.

\begin{lemma}
\label{lem:cox-sampled-event-noise}
Suppose~\eqref{eq:grouping-smallness} holds and
$\rho\kappa\le1/9$.  For the single-triple gradient $g_t$ in
\eqref{eq:main-beta-update}--\eqref{eq:main-shift-update}, the minimizer
$z_\rho$ of $G_{\rho,\delta}$ on
$\mathcal B\times[-M-1,M]^{K_T}$ satisfies
\begin{equation}
 \E_t\norm{g_t}_{P^{-1}}^2
 \le\frac{4A_{K_T}}\rho
 [G_{\rho,\delta}(z_t)-G_{\rho,\delta}(z_\rho)]+2V_T.
 \label{eq:cox-sampled-event-noise}
\end{equation}
\end{lemma}

\begin{proof}
Write $\bar\kappa_{\rm grp}$ for the maximum normalized exponential
moment of the group distributions.  Lemma~\ref{lem:group-kappa-envelope}
gives $\bar\kappa_{\rm grp}\le9\kappa/8$, so
$\rho\bar\kappa_{\rm grp}\le1/8$.  At $z_\rho$, the groupwise shift
condition and the moment bound used in Lemma~\ref{lem:smoothness-noise}
give $\E_kw_{k,*}=1$ and
$\E_kw_{k,*}^2\le4\bar\kappa_{\rm grp}$.  For the same sampled triple
at $z=(\beta,s)$ and $z_\rho$, decompose its gradient as $g(z)=b+r$, where
\begin{align*}
 b_\beta&=\lambda\beta+K_Tp_k(w_{k,*}x_J-x_I),
 &b_s&=K_Tp_k(1-w_{k,*})e_k,\\
 r&=K_Tp_k(w-w_{k,*})(x_J,-e_k).
\end{align*}
The bound $\norm\beta\le B_0$, $\norm{x_j}\le X$, and uniform group
sampling give
\begin{align*}
 \E\norm b_{P^{-1}}^2
 &\le2\lambda^2B_0^2
  +4K_TX^2\sum_kp_k^2(\E_kw_{k,*}^2+1)
  +K_T\sum_kp_k\E_k(1-w_{k,*})^2\\
 &\le2\lambda^2B_0^2
  +4K_TX^2(4\bar\kappa_{\rm grp}+1)
  +4K_T\bar\kappa_{\rm grp}
 \le V_T.
\end{align*}
The last step uses $1\le\bar\kappa_{\rm grp}\le9\kappa/8$.
As in the proof of Lemma~\ref{lem:smoothness-noise}, scalar
co-coercivity of $h_\rho$ yields
\[
 \E\norm r_{P^{-1}}^2
 \le\frac{K_T(X^2+1)}{2\rho}
 \sum_kp_k\E_k D_{h_\rho}(a,a_*).
\]
The sampled event term is affine in $\beta$, so its Bregman divergence
vanishes.  Consequently the sum on the right is at most
$D_{G_{\rho,\delta}}(z,z_\rho)$, which is at most the objective gap by
constrained optimality.  Apply
$\norm{b+r}_{P^{-1}}^2\le2\norm b_{P^{-1}}^2+2\norm r_{P^{-1}}^2$
and $K_T(X^2+1)\le4A_{K_T}$ to obtain
\eqref{eq:cox-sampled-event-noise}.
\end{proof}

\begin{theorem}
\label{thm:cox-computational}
Assume the standing covariate and predictor bounds and
\eqref{eq:exact-cox-curvature}.  Fix $T\ge2$ and $q\in(0,1)$, and suppose the chosen
$\delta_T$ satisfies
\[
 r_T:=\frac{\delta_T}{1+\delta_T},
 \qquad
 r_T\le1-q,
 \qquad
 L_qr_T\le\frac12.
\]
Set
\begin{equation}
 \rho_T:=\frac{8A_{K_T}\log T}{\mu T},
 \qquad
 \eta_T:=\frac{2\log T}{\mu T}.
 \label{eq:cox-algorithm-parameters}
\end{equation}
Suppose $\rho_T\le1$, $\rho_T\kappa\le1/9$, and
\begin{equation}
 \frac52X^2L_q^2r_T^2+20X^2\rho_T\kappa
 \le\frac\nu2.
 \label{eq:cox-algorithm-curvature}
\end{equation}
Run weighted projected softplus SGD for $T$ iterations and let
$\widehat\beta_T=\beta_{T+1}$.  Then
\begin{align}
 \E\norm{\widehat\beta_T-\beta^\star}^2
 \le{}&
 \frac{2D^2}{T^2}
 +\frac{8V_T\log T}{\mu^2T}
 \nonumber\\
 &+\frac{2X^2}{\nu^2}
 \left(L_q^2r_T^2+4\rho_T\kappa\right)^2.
 \label{eq:cox-mean-square-bound}
\end{align}
If $\beta^\star$ is in the interior of $\mathcal B$, then also
\begin{align}
 \E[\mathcal L(\widehat\beta_T)-\mathcal L(\beta^\star)]
 \le(\lambda+X^2)\left\{
 \frac{D^2}{T^2}
 +\frac{4V_T\log T}{\mu^2T}
 +\frac{X^2}{\nu^2}
 \left(L_q^2r_T^2+4\rho_T\kappa\right)^2
 \right\}.
 \label{eq:cox-objective-bound}
\end{align}
\end{theorem}

\begin{proof}
Proposition~\ref{prop:cox-total-comparison} and
\eqref{eq:cox-algorithm-curvature} show that, at the chosen $\rho_T$, the
profiled objective has Hessian at least $\nu I/2$.  Its group moment is at
most $9\kappa/8$.  Lemma~\ref{lem:generic-rsi} therefore gives the
joint restricted-secant constant $\mu$.  The sampled-event estimate
\eqref{eq:cox-sampled-event-noise}, unbiasedness, and the weighted
projection give the same recursion as in the proof of
Theorem~\ref{thm:main-optimization}, with $C=A_{K_T}$ and $V=V_T$.
Thus, for the profiled-surrogate minimizer
$\beta_{\rho_T,\delta_T}$,
\[
 \E\norm{\widehat\beta_T-\beta_{\rho_T,\delta_T}}^2
 \le\frac{D^2}{T^2}+\frac{4V_T\log T}{\mu^2T}.
\]
Theorem~\ref{thm:cox-transfer} bounds the squared distance from
$\beta_{\rho_T,\delta_T}$ to $\beta^\star$.  Apply
$\norm{a+b}^2\le2\norm a^2+2\norm b^2$ to obtain
\eqref{eq:cox-mean-square-bound}.

The exact objective is $(\lambda+X^2)$-smooth: each $\nabla^2a_i$ is a
covariance matrix bounded by $X^2I$.  If $\beta^\star$ is interior, then
$\nabla\mathcal L(\beta^\star)=0$, and smoothness gives
\[
 \mathcal L(\beta)-\mathcal L(\beta^\star)
 \le\frac{\lambda+X^2}{2}\norm{\beta-\beta^\star}^2.
\]
Combining this with~\eqref{eq:cox-mean-square-bound} proves
\eqref{eq:cox-objective-bound}.
\end{proof}

We now make the dependence on $T$ explicit.  Suppose every event risk set
has at least $cN$ subjects for a fixed $c\in(0,1)$.  Maximal grouping gives
\begin{equation}
 K_T
 \le\min\left\{m,
  \left\lceil\frac{\log(1/c)}{\log(1+\delta_T)}\right\rceil
 \right\}.
 \label{eq:cox-group-count}
\end{equation}
Indeed, let $n_k^{\rm start}$ be the first risk-set size in group $k$.
Maximality gives
$n_k^{\rm start}/n_{k+1}^{\rm start}>1+\delta_T$
for $k=1,\ldots,K_T-1$. Since
$n_1^{\rm start}/n_{K_T}^{\rm start}\le1/c$, multiplication yields
$(1+\delta_T)^{K_T-1}<1/c$, which is equivalent to
\eqref{eq:cox-group-count} after taking integer parts.

\begin{corollary}
\label{cor:cox-four-fifths}
Suppose every event risk set has at least $cN$ subjects for a fixed
$c\in(0,1)$, and suppose the constants in
Theorem~\ref{thm:cox-computational}, together with $c$, do not depend on
$T$.  For all sufficiently large $T$, choose
\begin{equation}
 \delta_T:=\left(\frac{\log T}{T}\right)^{1/5}
 \label{eq:optimal-delta}
\end{equation}
and use~\eqref{eq:cox-algorithm-parameters}.  Then
\begin{align}
 \E\norm{\widehat\beta_T-\beta^\star}^2
 &=O\!\left(\left(\frac{\log T}{T}\right)^{4/5}\right),
 \label{eq:four-fifths-square}\\
 \E[\mathcal L(\widehat\beta_T)-\mathcal L(\beta^\star)]
 &=O\!\left(\left(\frac{\log T}{T}\right)^{4/5}\right)
 \label{eq:four-fifths-objective}
\end{align}
when $\beta^\star$ is interior.  Moreover,
\begin{equation}
 \E\norm{\widehat\beta_T-\beta^\star}
 =O\!\left(\left(\frac{\log T}{T}\right)^{2/5}\right).
 \label{eq:two-fifths-norm}
\end{equation}
\end{corollary}

\begin{proof}
For $\delta_T\le1$,~\eqref{eq:cox-group-count} gives
$K_T=O(\delta_T^{-1})$.  Hence
\[
 A_{K_T}=O(\delta_T^{-1}),
 \qquad
 V_T=O(\delta_T^{-1}),
 \qquad
 \rho_T=O\!\left(\frac{\log T}{\delta_TT}\right).
\]
With~\eqref{eq:optimal-delta}, $\rho_T$ is
$O((\log T/T)^{4/5})$, so all the smallness and curvature conditions hold
eventually.  The stochastic term in~\eqref{eq:cox-mean-square-bound} is
\[
 O\!\left(\frac{\log T}{\delta_TT}\right)
 =O\!\left(\left(\frac{\log T}{T}\right)^{4/5}\right).
\]
The grouping part of the squared minimizer bias is
$O(\delta_T^4)$ and has the same order.  The squared softplus bias is
$O(\rho_T^2)$ and is smaller.  This proves
\eqref{eq:four-fifths-square}. The objective result follows from
Theorem~\ref{thm:cox-computational}.  Jensen's inequality gives
\eqref{eq:two-fifths-norm}.
\end{proof}

\begin{remark}
For a fixed finite data set, decreasing $\delta_T$ eventually produces
singleton event groups.  At that point $K_T\le m$ and the grouping bias is
zero, so the bound crosses over from the grouping-limited
$T^{-4/5}$ regime toward the ordinary $\widetilde O(m/T)$ finite-problem
rate.
\end{remark}

\begin{remark}
Suppose instead that $\beta^\star$ is unique and interior and only
\[
 \nabla^2\mathcal L(\beta^\star)\succeq\nu_*I
\]
is known.  Continuity gives a neighborhood on which
$\mathcal L(\beta)-\mathcal L(\beta^\star)
\ge(\nu_*/4)\norm{\beta-\beta^\star}^2$.  On the compact complement of
that neighborhood, uniqueness gives a strictly positive minimum objective
gap.  Combining the two regions shows that convergence in objective implies
convergence in argument, and an expected objective rate also gives the same
rate for expected squared distance.  This local observation does not by
itself establish the uniform profile curvature needed by the RSI during the
whole optimization path. That is why Theorem~\ref{thm:cox-computational}
uses~\eqref{eq:exact-cox-curvature} on $\mathcal B$.
\end{remark}

\section{Statistical consequence}
\label{sec:statistics}

This section proves Theorem~\ref{thm:main-statistics}.  To clarify its
minimum-risk-set condition, let $T^0$ and $C$ be a subject's event and
censoring times, and put $\widetilde T=\min(T^0,C)$,
$\Delta=\mathbf 1\{T^0\le C\}$, and
$Y(t)=\mathbf 1\{\widetilde T\ge t\}$.  Only observed failures at or before a
fixed horizon $t_{\max}<\infty$ enter the empirical partial likelihood.  If
$\inf_{0\le t\le t_{\max}}\Pr\{Y(t)=1\}>0$, a uniform law of large numbers
for independent subjects makes every risk set through this horizon contain a
fixed positive fraction of the sample with probability tending to one.
Independent subjects, the proportional-hazards model, conditionally
independent censoring, bounded covariates, and nonsingular population
information are familiar sufficient conditions for the full estimator's
classical limit \citep{tsiatis1981large,andersen1982cox}.  As in
Section~\ref{sec:statistics-main}, we take that limit as given.

\begin{proof}[Proof of Theorem~\ref{thm:main-statistics}]
Condition~\eqref{eq:main-statistical-condition} and $L_{q,N}\ge1$ imply
$\delta_N\to_p0$, $L_{q,N}\delta_N\to_p0$, and
$\rho_N\kappa_N\to_p0$.  Since $q$ is fixed, the smallness conditions of
Theorem~\ref{thm:main-compression} therefore hold with probability tending
to one.  On the same event, the segment joining the two estimators lies in
the convex neighborhood $U$.  Integrating the Hessian bound on that segment
and using the variational inequalities for the two constrained minimizers,
as in Theorem~\ref{thm:cox-transfer}, gives
\[
 \norm{\widehat\beta_N^{\rho,\delta}-\widehat\beta_N}
 \le\frac X{\nu_0}
 \left(L_{q,N}^2r_{\delta_N}^2+4\rho_N\kappa_N\right)
 \le\frac X{\nu_0}
 \left(L_{q,N}^2\delta_N^2+4\rho_N\kappa_N\right).
\]
After multiplication by $\sqrt N$, this tends to zero in probability.
Slutsky's theorem transfers the assumed limit of $\widehat\beta_N$ to
$\widehat\beta_N^{\rho,\delta}$.

For the schedule in Theorem~\ref{thm:main-statistics}, the two terms of
\eqref{eq:main-statistical-condition} are respectively
$O_p(\ell_N^{-2})$ and $O_p(\ell_N^{-1})$.  If every event risk set has at
least $cN$ subjects with probability tending to one, then
\eqref{eq:cox-group-count} and
$\log(1+\delta_N)\ge\delta_N/2$ eventually give
$K_N=O_p(\delta_N^{-1})=O_p(N^{1/4}\ell_N)=o_p(N)$ when
$\ell_N=o(N^{3/4})$.
\end{proof}

\paragraph{Vanishing ridge.}
A fixed ridge coefficient changes the population target.  If the exact and
grouped objectives share a coefficient $\lambda_N$ and the ridge and
unregularized estimators lie in $U$, the same curvature argument gives
$\norm{\widehat\beta_{N,\lambda_N}-\widehat\beta_{N,0}}
\le B_0\lambda_N/\nu_0$.  Thus $\sqrt N\lambda_N\to0$ preserves the
ordinary Cox limit.

\paragraph{Numerical accuracy and the two limits.}
Here $N$ is sample size, while $T$ is the iteration horizon.  The schedule in
Theorem~\ref{thm:main-statistics} controls statistical approximation as $N$
grows. Corollary~\ref{cor:cox-four-fifths} balances approximation and
computation as $T$ grows.  Once the statistical condition holds, a numerical
iterate $\widehat\beta_{N,T_N}$ retains the exact Cox limit if
$\norm{\widehat\beta_{N,T_N}-\widehat\beta_N^{\rho,\delta}}
=o_p(N^{-1/2})$.  A sufficient unconditional criterion is
$N\E\norm{\widehat\beta_{N,T_N}-\widehat\beta_N^{\rho,\delta}}^2\to0$.
When the grouped estimator and algorithm use the same balanced $\delta$ and
$\rho$, the global
computational assumptions hold with probability tending to one, and the
conditional error prefactor is $O_p(1)$, condition
\eqref{eq:main-end-to-end-statistics} suffices in probability.  An
unconditional mean-square conclusion additionally requires integrable
bounds, for example deterministic computational assumptions and a uniformly
integrable prefactor.

\input{experiments_appendix.en.tex}

\end{document}

%% file: tables/data.tex
\begin{table}[htbp]
\centering\footnotesize
\setlength{\tabcolsep}{4pt}
\caption{Datasets, training event counts $m_{\rm tr}$, group counts $K$ at $\delta=0.05$, and per-fit vector budgets (millions). Dimension $d$ is after preprocessing.}
\label{tab:study27-data}
\begin{tabular}{lrrrrrr}
\hline
Dataset & $N$ & $d$ & $m_{\rm tr}$ & $K$ & Tuning cap & Final cap \\
\hline
SUPPORT2 & 8,873 & 22 & 3,621 & 61 & 3 & 4.5 \\
NWTCO & 4,028 & 11 & 342 & 15 & 3 & 4.5 \\
Correlated 100k & 100,000 & 10 & 6,005 & 133 & 10 & 15 \\
Correlated 1M & 1,000,000 & 20 & 60,136 & 178 & 100 & 150 \\
Independent 100k & 100,000 & 20 & 41,921 & 166 & 10 & 15 \\
\hline
\end{tabular}
\end{table}

%% file: experiments_appendix.en.tex
\input{experimental_protocol.en.tex}
\input{experimental_diagnostics.en.tex}
\input{experimental_trajectories.en.tex}

%% file: experimental_protocol.en.tex
\section{Experimental protocol}
\label{sec:study27-protocol}

\subsection{Data preparation}

For SUPPORT2 \citep{rhoexpSupportData}, we use time to death and 14 covariates:
age, sex, race, number of comorbidities, diabetes, dementia, cancer status, mean blood pressure, heart rate, respiratory rate, temperature, white blood
cell count, sodium, and creatinine. Missing race is a category, records
missing another selected covariate or the outcome are excluded, retaining
8,873 of 9,105 subjects. Prognostic scores and outcome-derived variables
are excluded. For NWTCO \citep{rhoexpNwtcoData}, we use time to relapse and the corresponding
event indicator. Predictors are histology assessments
from the local institution and central laboratory, stage, study, and age.
Identifiers and the subcohort indicator
are excluded.

Synthetic event times follow
\[
 T=\{E\exp(-x^\top\beta_*)\}^{1/1.5},\qquad
 E\sim\operatorname{Exp}(1),\quad x\sim\mathcal N(0,\Sigma).
\]
Independently, we draw $b\sim\mathcal N(0,I_d)$ and set
$\beta_{\rm white}=0.5b/\|b\|_2$.
The correlated cases use $\Sigma_{uv}=r^{|u-v|}$, with $r=0.95$ for
$(N,d)=(10^5,10)$ and $r=0.90$ for $(10^6,20)$.
Let $Z\in\mathbb R^{N\times d}$ have independent standard normal entries.
We form $X=ZL^\top$, $LL^\top=\Sigma$, and set
$\beta_*=L^{-\top}\beta_{\rm white}$.
Independent 100k uses $(N,d)=(10^5,20)$, $\Sigma=I_{20}$, and
$\beta_*=\beta_{\rm white}$. Thus the population standard deviation
of the linear predictor is $0.5$ in every case.
Independent exponential censoring targets 90\% in the correlated cases
and 30\% in Independent 100k, calibrated on separate 65,536-observation
pilots; realized fractions are 89.991\%, 89.9774\%, and 30.131\%.
Data seeds are 1701 for both correlated cases and 2701 for the independent
case. The correlated cases differ in dimension and correlation, as well as $N$.

All methods share event-stratified 60/20/20 splits, using seed 1701 for
training and 1702 for splitting the remainder. Imputation medians, categorical vocabularies, and
standardization are fitted on training data. After one-hot encoding,
columns with training standard deviation at most $10^{-12}$ are removed.
Validation and test use the fitted transformation unchanged. Sorting
and risk indices are shared across all fits within each dataset and
training phase. Breslow ties share risk
sets while retaining event multiplicities in both the objective and
sampling distribution. Table~\ref{tab:study27-data} gives training event
counts and the resulting numbers of groups at $\delta=0.05$, alongside
dimensions and budgets.

Each dataset uses one realization and one fixed split. Dataset and
candidate-menu choices were informed by exploratory optimization results.

\subsection{Native objectives and optimizer conventions}
\label{sec:study27-baselines}

All methods use $x^\top\beta$, initialize $\beta=0$, and add the explicit
ridge penalty $\lambda\|\beta\|_2^2/2$, with $\lambda=10^{-3}$ and
zero optimizer weight decay. Their sampled risk sets and native training
objectives differ. Every trajectory is evaluated using the same
regularized full-Cox objective defined below.

\paragraph{Our method.}
Groups are sampled uniformly, each with probability $1/K$.
We use mini-batches of independent stochastic gradient samples,
the unbiased sampled event mean
$Kp_kx_i$ with $i\sim\operatorname{Unif}(\mathcal I_k)$, weighted projection
and arithmetic averaging of the projected post-update coefficient iterates. We project the coefficient vector onto the Euclidean ball
$\{\beta:\|\beta\|_2\le B_0\}$ and clip the updated auxiliary shifts
to $[-S,S]$. We fix $\rho=10^{-4}$, $\delta=0.05$, $B_0=10$, $S=40$,
and $t_0=1000$. 
Batch size and initial step are tuned. Full-Cox diagnostics never enter updates.

\paragraph{Batch LSE.}
Adam averages over $Q$ sampled events, each with $c$ controls from its
full training risk set, the loss
\[
 -z_i+\log\left(\frac1c\sum_{r=1}^c e^{z_{j_r}}\right),
 \qquad z_j=x_j^\top\beta,
\]
plus ridge. Taking the logarithm of a sampled mean generally gives a
biased full-Cox gradient.

\paragraph{Minibatch Cox.}
Adam uses event-normalized Cox loss with risk sets restricted to shuffled
observation batches of size $b$, retaining incomplete final batches.
A batch with no observed events has zero data gradient; the ridge term
and the optimizer's existing moment state remain active.

\paragraph{BigSurv.}
Updates average event-sum losses over $\ell$ strata of 20
observations \citep{rhoexpTarkhanSimon2020}. This data loss is divided by
$a_{\rm BS}=20m_{\rm tr}/n_{\rm tr}$, the expected number of observed
events in a uniformly sampled stratum, before adding ridge. Here
$n_{\rm tr}$ and $m_{\rm tr}$ are the training subject and event counts.
The
denominator is fixed from the training split rather than the realized
event count of each stratum. AMSGrad uses uncorrected moments
$(0.9,0.99)$, $\epsilon=10^{-8}$, $\eta_t=\eta_0/\sqrt t$ ($t\ge1$),
and arithmetic iterate averaging. Strata are formed from shuffled
training observations; incomplete strata batches are discarded at epoch
boundaries.

\paragraph{Cox-CC.}
Adam minimizes the case-control loss~\citep{kvamme2019time}
\[
 \frac1Q\sum_{i\in\text{sampled cases}}
 \log\left(1+\sum_{r=1}^c e^{z_{j_r}-z_i}\right)
\]
plus ridge, including the case score in the denominator. Batch LSE and
Cox-CC sample with replacement, allow controls equal to the case, and
do not deduplicate score evaluations.

Batch LSE, Minibatch Cox, and Cox-CC use Adam with moment decay rates
$(0.9,0.999)$, $\epsilon=10^{-8}$, bias correction, and a constant
selected learning rate. These local linear implementations use Breslow
ties and stable log-sum-exp. BigSurv jitter and Cox-CC score
clamping/additional shrinkage are disabled.

\subsection{Tuning menus and selection}

Twelve candidate configurations per method use tuning seeds 100 and 101:
\begin{equation}
 h_* = \arg\min_h\frac12\sum_{r\in\{100,101\}}
                 \min_{j\in\mathcal J_{h,r}}
                 \ell_{\rm val}(\beta_{h,r,j}),
 \label{eq:study27-hpo-selection}
\end{equation}
where $\mathcal J_{h,r}$ contains saved checkpoints; ties favor the lower
candidate index. Each method has two structural settings and six learning
rates per setting (Table~\ref{tab:study27-hpo}), giving 600 configuration--seed
evaluations across the five datasets and five methods. For each dataset,
all five configurations are fixed before its final fits. Each selected
configuration is trained from zero with optimizer seeds 0--9 on the same
fixed split, giving 250 final runs in total.
Per-fit caps are $U_{\rm HPO}=\max\{3\cdot10^6,100N\}$ and
$U_{\rm final}=1.5U_{\rm HPO}$, where $N$ is the retained pre-split
sample size. Budgets are enforced separately for each run.

\input{tables/hpo.tex}
\input{tables/selected.tex}
\FloatBarrier

\subsection{Vector-work accounting and timing}
\label{sec:study27-cost}

A vector unit is one length-$d$ inner product, accumulation, arithmetic
pass, or norm; $X\beta$ and $X^\top v$ each cost one per processed row.
Our updates cost three batch passes (risk scores, weighted risk and event
feature sums) and seven coefficient passes (gradient difference, ridge,
update, norm, and averaging); active projection adds one. Setup requires
no feature-vector arithmetic. Baseline Adam/AMSGrad use a fixed estimate
of 12 passes per update: total $2P+15t$ for $P$ physical scores and $t$
updates, plus $3t$ for BigSurv averaging. Units exclude scalar
nonlinearities, shifts, sampling, sorting, searches, memory copies, and
runtime overhead; this work enters elapsed time when performed within
the timed intervals below. Units model vector work, not hardware FLOPs.

All runs use float64 arithmetic, one compute thread, and one training
worker. Our method uses NumPy and the baselines
use PyTorch. Reported time is method setup plus optimizer execution,
including sampling and budget checks. Common data preparation,
reference computation, checkpoint snapshot recording, full-data
diagnostics, and serialization are excluded. Runtime comparisons
characterize these implementations and were recorded in separate execution
sessions; software dependencies are listed in the code repository.

\subsection{Metrics, checkpoints, and aggregation}

For split $D$ with $m_D$ events,
\[
 \ell_D(\beta)=\frac1{m_D}\sum_{i\in\mathcal E_D}
 \left[\log\sum_{j\in\mathcal R_{i,D}}e^{x_j^\top\beta}
       -x_i^\top\beta\right].
\]
On the training split, $\ell_{\rm tr}(\beta)+\lambda\norm\beta^2/2$
differs from $\mathcal L(\beta)$ in~\eqref{eq:main-exact-cox} only by
$m^{-1}\sum_{i\in\mathcal E}\log n_i$. We report the signed gap
$\mathcal L(\beta)-\mathcal L(\beta_{\rm ref})$, with recorded resolution
approximately $1.2$--$1.8\cdot10^{-13}$. References minimize the
full-risk-set, event-normalized Breslow objective with $\lambda=10^{-3}$,
using float64 L-BFGS/strong-Wolfe from zero and a Newton-polishing
fallback. All five stopped at full regularized gradient norm below
$10^{-10}$. Reference computation is separate from compared trajectories.

Validation/test use unpenalized $\ell_D$ and reporting-only Harrell
C-index. Each tuning and final run uses 48 geometrically spaced work
targets from 128 to its respective cap, together with the initial and
actual terminal states. We record the first completed update crossing
each target and coalesce targets crossed by the same update. Target
locations are shared within a dataset and phase, while realized
checkpoint costs can differ between methods. The saved checkpoint with
the smallest validation Cox loss is selected, with ties resolved in favor
of the earliest checkpoint. Full-data diagnostics are computed after
training and do not affect updates.

A run attains a threshold through the cap at the earliest saved post-update
checkpoint at which the gap plus reference resolution, and its value at
every later saved checkpoint, is at most that threshold. We use thresholds
$10^{-3},10^{-5},10^{-7},10^{-9}$. A first crossing at the terminal
checkpoint also qualifies. Costs use recorded points without
interpolation; nonattainment is reported separately from the median cost
among attainers. This assesses saved checkpoints rather than all iterates.

Figures show ten seed traces, medians, and 25th--75th percentile bands.
Aggregation follows linear interpolation in displayed axis coordinates
on common observed support, without extrapolation or best-so-far
smoothing. Tables report means and sample SD where stated. Variability
concerns optimizer randomness on one split, not independent cohorts or
realizations.

For each dataset, all 50 validation-selected models are fixed before
their test evaluation. Each selected model is scored once on its test split. These test cohorts were also evaluated during earlier
exploratory work, so the predictive results are descriptive rather than
a new independent confirmation.

\FloatBarrier
\section{Additional results}
\label{sec:study27-additional-results}

\subsection{Optimization accuracy and target attainment}

Our method achieves the lowest median terminal full-Cox training gap
on all five datasets at matched vector-work caps
(Table~\ref{tab:study27-terminal}). BigSurv is the closest baseline,
with median gaps 3.74, 2.38, 16.14, 26.15, and 1.22 times ours in table
order. Our method has a smaller gap than BigSurv for all ten seed labels
on four datasets and for eight on Independent 100k.
These results compare complete training
procedures against a common regularized full-Cox reference; their native
objectives and optimizers differ
(Appendix~\ref{sec:study27-baselines}).

\input{tables/terminal.tex}

At the $10^{-5}$ target, our method is the only method to attain the
required accuracy within the final caps: seven of ten runs on
Correlated 100k and all ten on Correlated 1M
(Table~\ref{tab:study27-attainment}). At $10^{-3}$, our method and
BigSurv attain the target in all ten runs on every dataset. Among these
two methods, our method uses less median work on Correlated 1M, while
BigSurv uses less on SUPPORT2, NWTCO, Correlated 100k, and Independent
100k. No method attains $10^{-7}$ or $10^{-9}$ within these caps.
Reported costs are conditional on sustained attainment at saved
checkpoints through the cap and depend on the target accuracy.

\input{tables/attainment.tex}
\clearpage
\begingroup
\setlength{\intextsep}{3pt}
\setlength{\floatsep}{6pt}
\setlength{\textfloatsep}{6pt}
\setlength{\abovecaptionskip}{3pt}

\subsection{Predictive performance}

Tables~\ref{tab:study27-test} and~\ref{tab:study27-validation} report
validation-selected metrics. Mean test Cox loss is lowest for our method
on NWTCO, Correlated 100k, and Independent 100k, for Minibatch Cox on
SUPPORT2, and for BigSurv on Correlated 1M. Mean test C-index is highest
for our method on NWTCO and Correlated 100k, for BigSurv on SUPPORT2 and
Independent 100k, and for Cox-CC on Correlated 1M.

\input{tables/test.tex}
\input{tables/validation.tex}
\FloatBarrier
\endgroup

%% file: tables/hpo.tex
\begin{table}[htbp]
\centering
\footnotesize
\setlength{\tabcolsep}{4pt}
\caption{Twelve candidate configurations per method: two structural settings and six learning rates per setting, shared across datasets. Here $b$ is batch size, $Q$ is the number of sampled cases, $c$ the controls per case, and $\ell$ the strata per update.}
\label{tab:study27-hpo}
\begin{tabular}{lp{0.70\linewidth}}
\hline
Method & Candidates \\
\hline
Ours & $(b,\eta_0)\in\{128,256\}\times\{0.003,0.01,0.03,0.1,0.3,1\}$ \\
Batch LSE & $(Q,c)=(8,8)$: $\eta\in\{0.0009,0.003,0.009,0.03,0.09,0.3\}$; \newline $(Q,c)=(16,64)$: $\eta\in\{0.0003,0.001,0.003,0.01,0.03,0.1\}$ \\
Minibatch Cox & $(b,\eta)\in\{32,256\}\times\{0.001,0.003,0.01,0.03,0.1,0.3\}$ \\
BigSurv & Stratum size 20; $(\ell,\eta_0)\in\{1,16\}\times\{0.006,0.02,0.06,0.2,0.6,2\}$ \\
Cox-CC & $Q=32$; $(c,\eta)\in\{1,8\}\times\{0.0003,0.001,0.003,0.01,0.03,0.1\}$ \\
\hline
\end{tabular}
\end{table}

%% file: tables/selected.tex
\begin{table}[htbp]
\centering\footnotesize
\setlength{\tabcolsep}{4pt}
\caption{Configurations minimizing the mean, over two tuning seeds,
of the minimum saved validation Cox loss within the tuning budget.
Tuples are $(b,\eta_0)$ for Ours, $(Q,c,\eta)$ for Batch LSE,
$(b,\eta)$ for Minibatch Cox, $(\ell,\eta_0)$ for BigSurv, and
$(c,\eta)$ for Cox-CC (with $Q=32$). Candidate menus are given in
Table~\ref{tab:study27-hpo}.}
\label{tab:study27-selected}
\begin{tabular}{lrrrrr}
\hline
Method & SUPPORT2 & NWTCO & Correlated 100k & Correlated 1M & Independent 100k \\
\hline
Ours & $(256,0.03)$ & $(128,0.03)$ & $(128,0.1)$ & $(256,0.1)$ & $(128,0.03)$ \\
Batch LSE & $(16,64,0.001)$ & $(16,64,0.01)$ & $(8,8,0.009)$ & $(16,64,0.0003)$ & $(16,64,0.0003)$ \\
Minibatch Cox & $(256,0.01)$ & $(256,0.03)$ & $(32,0.003)$ & $(256,0.001)$ & $(256,0.001)$ \\
BigSurv & $(1,0.02)$ & $(16,0.06)$ & $(1,2)$ & $(1,2)$ & $(16,0.06)$ \\
Cox-CC & $(8,0.001)$ & $(8,0.03)$ & $(8,0.03)$ & $(1,0.0003)$ & $(8,0.0003)$ \\
\hline
\end{tabular}
\end{table}

%% file: tables/terminal.tex
\begin{table}[htbp]
\centering\footnotesize
\setlength{\tabcolsep}{4pt}
\caption{Median terminal full-Cox training gap over ten optimizer seeds}
\label{tab:study27-terminal}
\begin{tabular}{lrrrrr}
\hline
Dataset & Ours & Batch LSE & Minibatch Cox & BigSurv & Cox-CC \\
\hline
SUPPORT2 & $\mathbf{3.93\cdot10^{-5}}$ & $1.47\cdot10^{-3}$ & $2.13\cdot10^{-3}$ & $1.47\cdot10^{-4}$ & $1.16\cdot10^{-3}$ \\
NWTCO & $\mathbf{4.96\cdot10^{-5}}$ & $7.77\cdot10^{-3}$ & $2.98\cdot10^{-3}$ & $1.18\cdot10^{-4}$ & $1.19\cdot10^{-2}$ \\
Correlated 100k & $\mathbf{7.92\cdot10^{-6}}$ & $1.76\cdot10^{-2}$ & $1.61\cdot10^{-3}$ & $1.28\cdot10^{-4}$ & $1.24\cdot10^{-2}$ \\
Correlated 1M & $\mathbf{1.39\cdot10^{-6}}$ & $4.79\cdot10^{-4}$ & $7.11\cdot10^{-4}$ & $3.64\cdot10^{-5}$ & $6.46\cdot10^{-4}$ \\
Independent 100k & $\mathbf{1.36\cdot10^{-5}}$ & $4.49\cdot10^{-4}$ & $2.27\cdot10^{-4}$ & $1.65\cdot10^{-5}$ & $2.68\cdot10^{-4}$ \\
\hline
\end{tabular}
\end{table}

%% file: tables/attainment.tex
\begin{table}[htbp]
\centering\footnotesize
\setlength{\tabcolsep}{4pt}
\caption{Sustained training-gap attainment at saved post-update checkpoints
through the final cap. Threshold columns give the number of attainers
out of ten runs. Costs are median work in millions of vector units
among attainers. A first crossing at the terminal checkpoint qualifies,
a dash indicates no attainment.}
\label{tab:study27-attainment}
\begin{tabular}{llrrrr}
\hline
Dataset & Method & $10^{-3}$ & Cost (M) & $10^{-5}$ & Cost (M) \\
\hline
SUPPORT2 & Ours & 10/10 & 0.3115 & 0/10 & -- \\
SUPPORT2 & Batch LSE & 1/10 & 4.4980 & 0/10 & -- \\
SUPPORT2 & Minibatch Cox & 0/10 & -- & 0/10 & -- \\
SUPPORT2 & BigSurv & 10/10 & 0.0654 & 0/10 & -- \\
SUPPORT2 & Cox-CC & 3/10 & 4.4999 & 0/10 & -- \\
NWTCO & Ours & 10/10 & 0.3108 & 0/10 & -- \\
NWTCO & Batch LSE & 0/10 & -- & 0/10 & -- \\
NWTCO & Minibatch Cox & 1/10 & 4.4998 & 0/10 & -- \\
NWTCO & BigSurv & 10/10 & 0.0474 & 0/10 & -- \\
NWTCO & Cox-CC & 0/10 & -- & 0/10 & -- \\
Correlated 100k & Ours & 10/10 & 0.9767 & 7/10 & 11.7018 \\
Correlated 100k & Batch LSE & 0/10 & -- & 0/10 & -- \\
Correlated 100k & Minibatch Cox & 4/10 & 12.0642 & 0/10 & -- \\
Correlated 100k & BigSurv & 10/10 & 0.4127 & 0/10 & -- \\
Correlated 100k & Cox-CC & 0/10 & -- & 0/10 & -- \\
Correlated 1M & Ours & 10/10 & 0.9571 & 10/10 & 25.1968 \\
Correlated 1M & Batch LSE & 10/10 & 61.4778 & 0/10 & -- \\
Correlated 1M & Minibatch Cox & 6/10 & 72.6711 & 0/10 & -- \\
Correlated 1M & BigSurv & 10/10 & 1.2885 & 0/10 & -- \\
Correlated 1M & Cox-CC & 9/10 & 13.9025 & 0/10 & -- \\
Independent 100k & Ours & 10/10 & 0.3621 & 0/10 & -- \\
Independent 100k & Batch LSE & 10/10 & 5.5559 & 0/10 & -- \\
Independent 100k & Minibatch Cox & 10/10 & 0.2205 & 0/10 & -- \\
Independent 100k & BigSurv & 10/10 & 0.0388 & 0/10 & -- \\
Independent 100k & Cox-CC & 10/10 & 1.2523 & 0/10 & -- \\
\hline
\end{tabular}
\end{table}

%% file: tables/test.tex
\begin{table}[!ht]
\centering\footnotesize
\setlength{\tabcolsep}{4pt}
\caption{Test metrics at validation-selected checkpoints: mean $\pm$ sample SD over ten optimizer seeds. One fixed split; unpenalized Cox loss per event ($\downarrow$), Harrell C-index ($\uparrow$).}
\label{tab:study27-test}
\begin{tabular}{llrr}
\hline
Dataset & Method & Test Cox loss & Test C-index \\
\hline
SUPPORT2 & Ours & $6.866840\pm0.000700$ & $0.602019\pm0.000693$ \\
SUPPORT2 & Batch LSE & $6.867794\pm0.001330$ & $0.600758\pm0.001080$ \\
SUPPORT2 & Minibatch Cox & $6.866488\pm0.002054$ & $0.601296\pm0.001637$ \\
SUPPORT2 & BigSurv & $6.867411\pm0.000304$ & $0.602970\pm0.000435$ \\
SUPPORT2 & Cox-CC & $6.867097\pm0.001045$ & $0.601770\pm0.001126$ \\
NWTCO & Ours & $6.227649\pm0.001645$ & $0.731458\pm0.000142$ \\
NWTCO & Batch LSE & $6.241507\pm0.015734$ & $0.726018\pm0.007678$ \\
NWTCO & Minibatch Cox & $6.248697\pm0.012824$ & $0.723723\pm0.004764$ \\
NWTCO & BigSurv & $6.232977\pm0.005778$ & $0.730110\pm0.001328$ \\
NWTCO & Cox-CC & $6.263264\pm0.020098$ & $0.717591\pm0.008785$ \\
Correlated 100k & Ours & $8.420984\pm0.000132$ & $0.635188\pm0.000056$ \\
Correlated 100k & Batch LSE & $8.424820\pm0.001777$ & $0.634273\pm0.001613$ \\
Correlated 100k & Minibatch Cox & $8.422421\pm0.001239$ & $0.634502\pm0.000873$ \\
Correlated 100k & BigSurv & $8.421654\pm0.000315$ & $0.634838\pm0.000220$ \\
Correlated 100k & Cox-CC & $8.422216\pm0.001405$ & $0.634370\pm0.000942$ \\
Correlated 1M & Ours & $10.702666\pm0.000042$ & $0.639196\pm0.000023$ \\
Correlated 1M & Batch LSE & $10.702925\pm0.000225$ & $0.639051\pm0.000096$ \\
Correlated 1M & Minibatch Cox & $10.702853\pm0.000179$ & $0.639099\pm0.000066$ \\
Correlated 1M & BigSurv & $10.702654\pm0.000035$ & $0.639178\pm0.000019$ \\
Correlated 1M & Cox-CC & $10.702712\pm0.000103$ & $0.639204\pm0.000037$ \\
Independent 100k & Ours & $8.725094\pm0.000077$ & $0.630792\pm0.000043$ \\
Independent 100k & Batch LSE & $8.725489\pm0.000289$ & $0.630621\pm0.000188$ \\
Independent 100k & Minibatch Cox & $8.725183\pm0.000084$ & $0.630782\pm0.000090$ \\
Independent 100k & BigSurv & $8.725162\pm0.000225$ & $0.630811\pm0.000022$ \\
Independent 100k & Cox-CC & $8.725228\pm0.000211$ & $0.630728\pm0.000135$ \\
\hline
\end{tabular}
\end{table}

%% file: tables/validation.tex
\begin{table}[!ht]
\centering\footnotesize
\setlength{\tabcolsep}{4pt}
\caption{Validation metrics at checkpoints selected by validation Cox loss: mean $\pm$ sample SD over ten optimizer seeds. One fixed split; unpenalized Cox loss per event ($\downarrow$), Harrell C-index ($\uparrow$).}
\label{tab:study27-validation}
\begin{tabular}{llrr}
\hline
Dataset & Method & Validation Cox loss & Validation C-index \\
\hline
SUPPORT2 & Ours & $6.871160\pm0.000429$ & $0.594114\pm0.000440$ \\
SUPPORT2 & Batch LSE & $6.871841\pm0.000956$ & $0.593927\pm0.001286$ \\
SUPPORT2 & Minibatch Cox & $6.870807\pm0.000519$ & $0.594091\pm0.000511$ \\
SUPPORT2 & BigSurv & $6.871334\pm0.000152$ & $0.594797\pm0.000370$ \\
SUPPORT2 & Cox-CC & $6.870968\pm0.000646$ & $0.594495\pm0.000652$ \\
NWTCO & Ours & $6.153779\pm0.000821$ & $0.723944\pm0.001528$ \\
NWTCO & Batch LSE & $6.141673\pm0.003232$ & $0.728380\pm0.002980$ \\
NWTCO & Minibatch Cox & $6.139356\pm0.005738$ & $0.728560\pm0.003219$ \\
NWTCO & BigSurv & $6.141284\pm0.005943$ & $0.727106\pm0.001478$ \\
NWTCO & Cox-CC & $6.135449\pm0.005278$ & $0.731840\pm0.002403$ \\
Correlated 100k & Ours & $8.402429\pm0.000093$ & $0.642029\pm0.000081$ \\
Correlated 100k & Batch LSE & $8.404346\pm0.001473$ & $0.641605\pm0.001088$ \\
Correlated 100k & Minibatch Cox & $8.402208\pm0.000598$ & $0.641995\pm0.000268$ \\
Correlated 100k & BigSurv & $8.402015\pm0.000557$ & $0.642109\pm0.000269$ \\
Correlated 100k & Cox-CC & $8.402560\pm0.000860$ & $0.641725\pm0.000603$ \\
Correlated 1M & Ours & $10.705947\pm0.000033$ & $0.637771\pm0.000027$ \\
Correlated 1M & Batch LSE & $10.706178\pm0.000166$ & $0.637626\pm0.000096$ \\
Correlated 1M & Minibatch Cox & $10.706057\pm0.000099$ & $0.637729\pm0.000065$ \\
Correlated 1M & BigSurv & $10.705949\pm0.000016$ & $0.637750\pm0.000011$ \\
Correlated 1M & Cox-CC & $10.705979\pm0.000040$ & $0.637758\pm0.000047$ \\
Independent 100k & Ours & $8.711592\pm0.000117$ & $0.634014\pm0.000091$ \\
Independent 100k & Batch LSE & $8.711744\pm0.000291$ & $0.633971\pm0.000185$ \\
Independent 100k & Minibatch Cox & $8.711506\pm0.000052$ & $0.634057\pm0.000052$ \\
Independent 100k & BigSurv & $8.711569\pm0.000096$ & $0.634021\pm0.000090$ \\
Independent 100k & Cox-CC & $8.711505\pm0.000148$ & $0.634071\pm0.000087$ \\
\hline
\end{tabular}
\end{table}

%% file: experimental_diagnostics.en.tex
\clearpage
\section{Approximation diagnostics}
\label{sec:study27-diagnostics}

On each training dataset we evaluate
$v\in\{0,0.5\beta_{\rm ref},\beta_{\rm ref},1.5\beta_{\rm ref}\}$,
where $\beta_{\rm ref}$ is the saved regularized full-Cox reference.
The grouping grid is $\delta\in\{0.2,0.1,0.05,0.025,0.0125\}$;
the softplus grid is
$\rho\in\{10^{-2},3\cdot10^{-3},10^{-3},3\cdot10^{-4},10^{-4},3\cdot10^{-5}\}$
at $\delta=0.05$.

We use $\mathcal L$, $\widetilde{\mathcal L}_\delta$, and
$\mathcal L_{\rho,\delta}$ from Section~\ref{sec:cox-main}, all with
$\lambda=10^{-3}$.  The grouped objective uses exact group normalizers;
the softplus objective profiles the shifts over $\mathbb R^K$.
The errors are Euclidean coefficient-gradient differences:
\begin{align}
 e_{\rm group}(v,\delta)
 &=\|\nabla\widetilde{\mathcal L}_\delta(v)
       -\nabla\mathcal L(v)\|,\\
 e_{\rm soft}(v,\rho)
 &=\|\nabla\mathcal L_{\rho,0.05}(v)
       -\nabla\widetilde{\mathcal L}_{0.05}(v)\|.
\end{align}
Softplus shifts are profiled by safeguarded scalar roots with tolerance
$5\cdot10^{-13}$. All 100 grouping and 120 softplus comparisons are
retained; near-zero errors at $v=0$ are excluded from logarithmic fits.

Figure~\ref{fig:study27-diagnostics} and Table~\ref{tab:study27-slopes}
use every grid point at $\beta_{\rm ref}$. Ordinary least-squares
log--log slopes are 1.37--1.87 for grouping and 0.997--1.000 for
softplus error. Softplus error is nearly linear in $\rho$ over the
tested range. Grouping changes discretely as $\delta$ varies, so its
finite-grid slopes also reflect changes in the partition.
The theoretical $O(\delta^2)$ bound applies under the stated compression
conditions and does not prescribe an exact finite-grid slope.
These diagnostics characterize approximation at fixed coefficients;
they do not estimate convergence over optimizer iterations.

\input{tables/diagnostic_slopes.tex}

The markers evaluate $r_\delta\le1-q$, $L_q(v)r_\delta\le1/2$,
and $\rho\kappa(v)\le1/9$, with $q=0.5$ and
$r_\delta=\delta/(1+\delta)$. Here $L_q(v)$ and $\kappa(v)$ are the
expressions in~\eqref{eq:main-Lq} and~\eqref{eq:main-kappa}
evaluated at $\beta=v$ before taking the supremum over $\mathcal B$.
The applicable conditions hold in 80/100 grouping and 96/120 softplus
comparisons. These pointwise checks do not certify the assumptions
uniformly over the radius-10 ball. The calculations preserve event
multiplicity and the Breslow convention for ties.

%% file: tables/diagnostic_slopes.tex
\begin{table}[htbp]
\centering\footnotesize
\setlength{\tabcolsep}{4pt}
\caption{Group count $K$ at $\delta=0.05$ and ordinary least-squares
log--log slopes of coefficient-gradient error against $\delta$
(grouping) and $\rho$ (softplus), using every point in the respective
grid at $\beta_{\rm ref}$.}
\label{tab:study27-slopes}
\begin{tabular}{lrrr}
\hline
Dataset & $K$ & $\delta$ slope & $\rho$ slope \\
\hline
SUPPORT2 & 61 & 1.783 & 1.000 \\
NWTCO & 15 & 1.601 & 0.997 \\
Correlated 100k & 133 & 1.372 & 0.999 \\
Correlated 1M & 178 & 1.716 & 0.999 \\
Independent 100k & 166 & 1.866 & 0.999 \\
\hline
\end{tabular}
\end{table}

%% file: experimental_trajectories.en.tex
\clearpage
\section{Complete trajectory figures}
\label{sec:study27-figures}

\begin{figure}[!ht]
\centering
\includegraphics[scale=0.76]{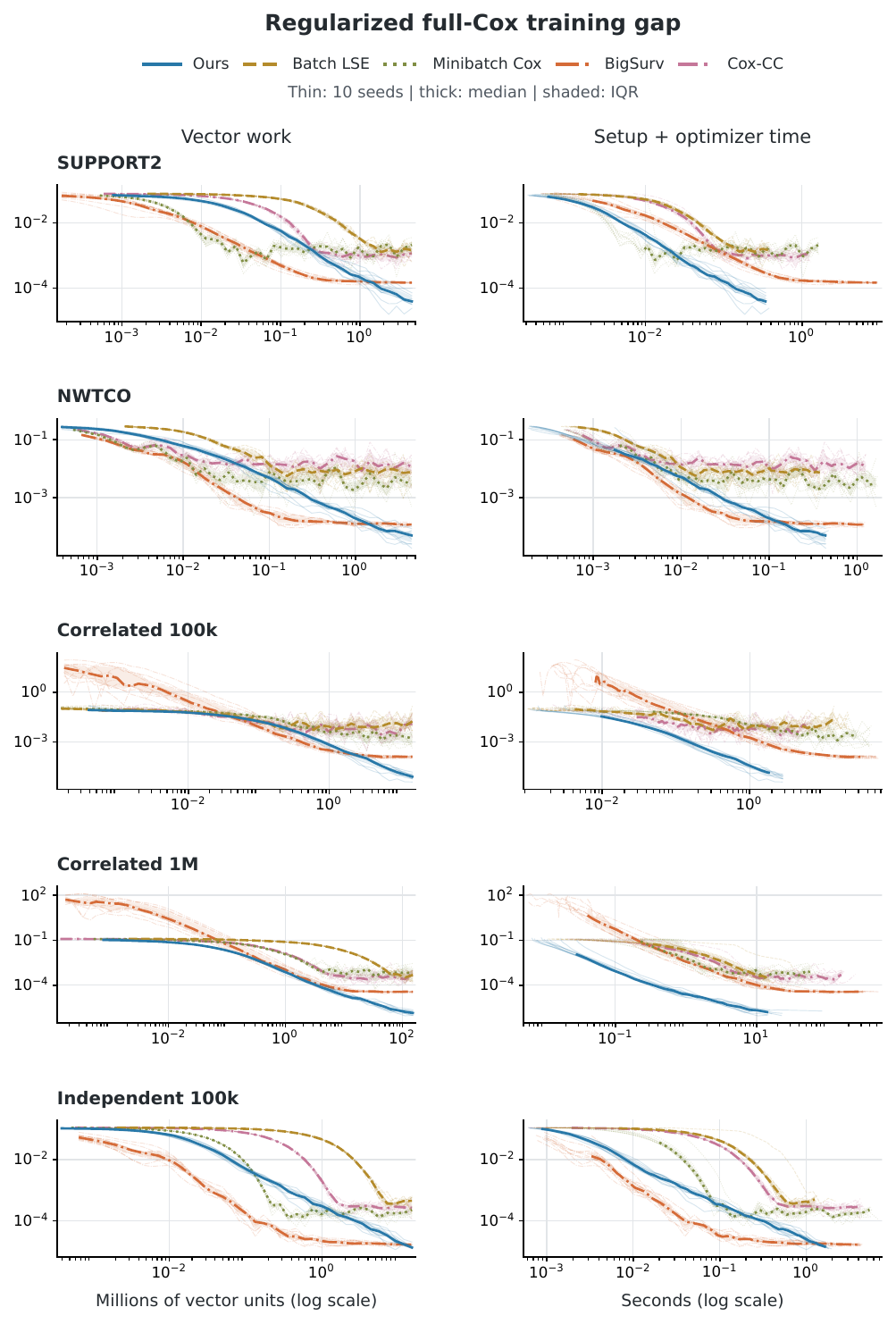}
\caption{Regularized full-Cox training gap against vector work and setup plus optimizer time. Thin lines show individual optimizer seeds, thick lines and bands give the median and interquartile range. Ours and BigSurv use running arithmetic averages of post-update
coefficient iterates; Ours averages after projection. The other methods
use current coefficient iterates.}
\label{fig:study27-trajectories-gap}
\end{figure}

\clearpage
\begin{figure}[p]
\centering
\includegraphics[width=\linewidth]{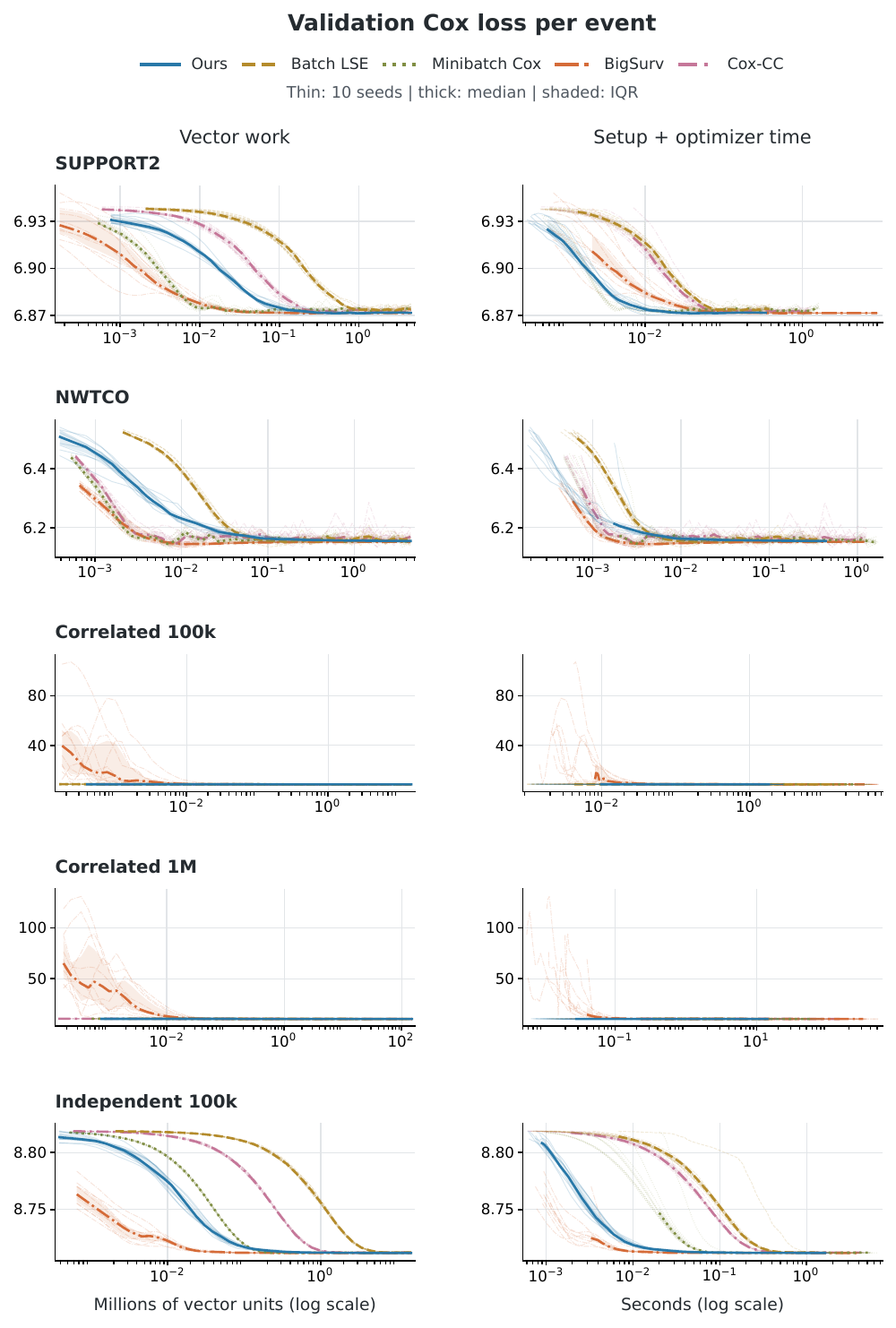}
\caption{Unpenalized validation Cox loss per event against vector work (left) and setup plus optimizer time (right); lower is better. Matched panels share Y-axis limits.}
\label{fig:study27-trajectories-validation-loss}
\end{figure}

\clearpage
\begin{figure}[p]
\centering
\includegraphics[width=\linewidth]{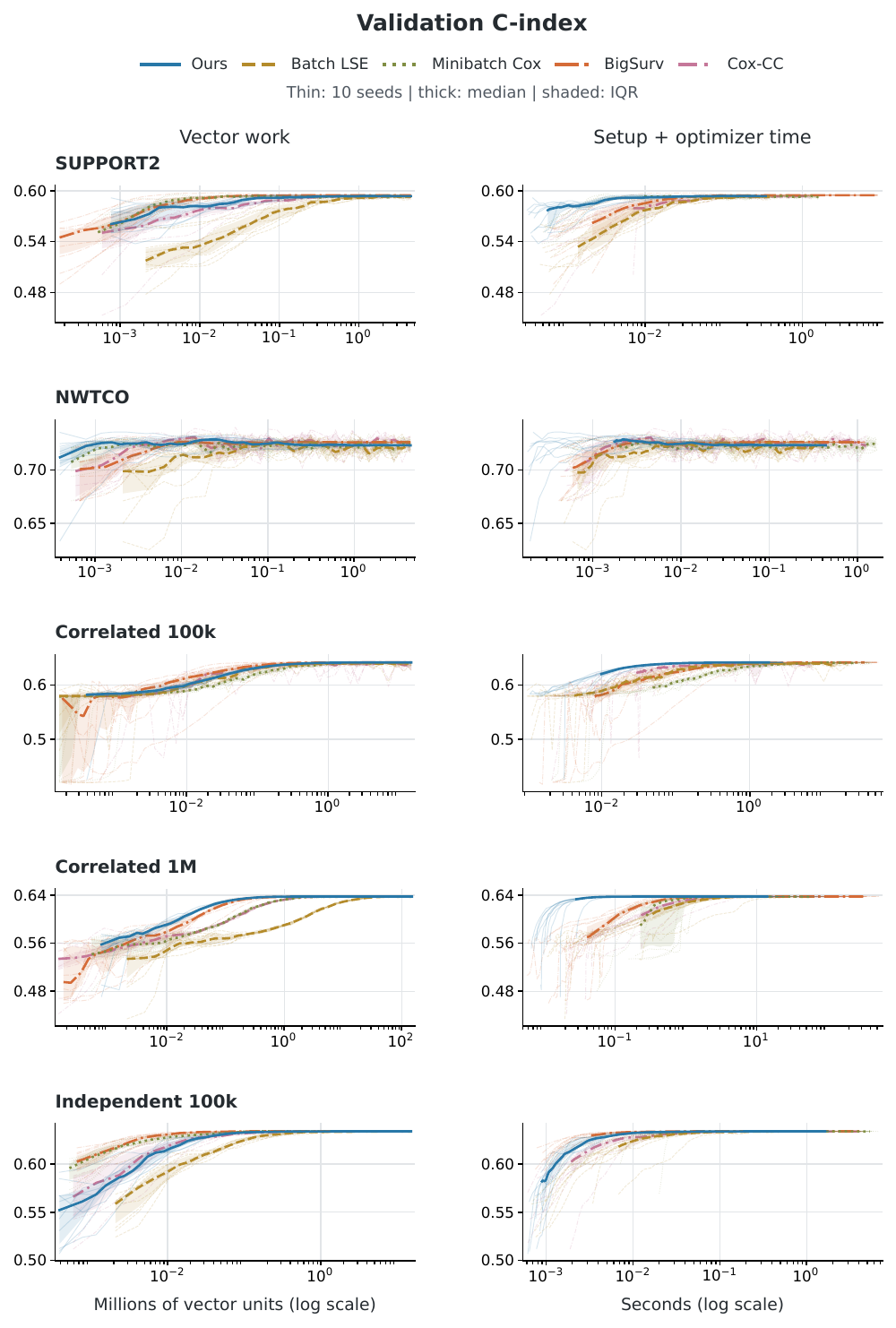}
\caption{Validation Harrell C-index against vector work (left) and
setup plus optimizer time (right); higher is better.
Matched panels share Y-axis limits.}
\label{fig:study27-trajectories-validation-cindex}
\end{figure}